\documentclass[11pt,oneside]{article}

\usepackage{style_mjh_forarXiv}
\usepackage[numbers]{natbib}
\usepackage{amsmath}                                          
\usepackage{amsthm}                                           
\theoremstyle{definition} \newtheorem{defn}{Definition}
\theoremstyle{plain} \newtheorem{prop}[defn]{Proposition}           
\theoremstyle{plain} \newtheorem{thm}[defn]{Theorem}                
\theoremstyle{plain} \newtheorem{lem}[defn]{Lemma}                  
\theoremstyle{plain}               
\theoremstyle{remark} \newtheorem{rmk}[defn]{Remark}                
\theoremstyle{remark}                 

\usepackage{enumitem}
\makeatletter
\def\namedlabel#1#2{\begingroup
    #2%
    \def\@currentlabel{#2}%
    \phantomsection\label{#1}\endgroup
}
\makeatother

\usepackage[pdfauthor={Matthew James Holland},%
pdftitle={holland2018MarginNewLoss},%
colorlinks=true,%
linkcolor=blue,%
citecolor=blue,%
pdftex]{hyperref} 

\begin{document}

\title{\textbf{Classification using margin pursuit}}

\author{
  Matthew J.~Holland\thanks{Email: \texttt{matthew-h@ids.osaka-u.ac.jp}.}\\
  Osaka University\\
  Yamada-oka 2-8, Suita, Osaka, Japan
}
\date{} 

\maketitle

\begin{abstract}
In this work, we study a new approach to optimizing the margin distribution realized by binary classifiers. The classical approach to this problem is  simply maximization of the expected margin, while more recent proposals consider simultaneous variance control and proxy objectives based on robust location estimates, in the vein of keeping the margin distribution sharply concentrated in a desirable region. While conceptually appealing, these new approaches are often computationally unwieldy, and theoretical guarantees are limited. Given this context, we propose an algorithm which searches the hypothesis space in such a way that a pre-set ``margin level'' ends up being a distribution-robust estimator of the margin location. This procedure is easily implemented using gradient descent, and admits finite-sample bounds on the excess risk under unbounded inputs. Empirical tests on real-world benchmark data reinforce the basic principles highlighted by the theory, and are suggestive of a promising new technique for classification.
\end{abstract}


\section{Introduction}\label{sec:intro}

Machine learning systems depend on both statistical inference procedures and efficient implementations of these procedures. These issues are reflected clearly within a risk minimization framework, in which given a known loss $\loss(\ww;\zz)$ depending on data $\zz$ and parameters $\ww$, the ultimate objective is minimization of the risk $\risk(\ww) \defeq \exx\loss(\ww;\zz)$, where expectation is taken with respect to the data. Since $\risk$ is unknown, the learner seeks to determine a candidate $\wwhat$ based on a limited sample $\zz_{1},\ldots,\zz_{n}$ such that $\risk(\wwhat)$ is sufficiently small, with high probability over the random draw of the sample. Inference is important because $R$ is always unknown, and the implementation is important because the only $\wwhat$ we ever have in practice is one that can be computed given finite time, memory, and processing power.

Our problem of interest is binary classification, where $\zz = (\xx,y)$ with inputs $\xx \in \RR^{d}$ and labels $y \in \{-1,1\}$. Parameter $\ww$ shall determine a scoring rule $h(\cdot;\ww)$, where $h(\xx) > 0$ implies a prediction of $y=+1$, and $h(\xx) \leq 0$ implies a prediction of $y=-1$. The classification \textit{margin} achieved by such a candidate is $y \, h(\xx)$, and the importance of the margin in terms of evaluating algorithm performance has been recognized for many years \citep{anthony1999NNTheory,langford2002a}. The work of \citet{koltchinskii2002a} provide risk bounds that depend on the empirical mean of $I\{y\,h(\xx) \leq \gamma\}$, providing useful generalization bounds for existing procedures whose on-sample margin error can be controlled. Intuitively, one might expect that having larger minimum margins on average would lead to better off-sample generalization. However, influential work by \citet{breiman1999a} showed that the problem is not so simple, demonstrating cases in which the margins achieved are higher, but generalization is worse. In response to this, \citet{reyzin2006a} make the important suggestion that it is not merely the location of the margins, but properties of the entire \textit{margin distribution} that are important to generalization.

New algorithms based on trying to control the empirical margin distribution, albeit indirectly, were proposed early on by \citet{garg2003a}, who proposed a strategy of optimizing the random projection error, namely $\prr\{ h(\xx)\,\widetilde{h}(\widetilde{\xx}) < 0 \}$, where $\widetilde{h}$ and $\widetilde{\xx}$ are respectively random projections of $h$ and $\xx$ from $d$-dimensional Euclidean space to a $k$-dimensional subspace, where $k \ll d$. The bounds are lucid and are suggestive of practical objective functions, but their analysis requires that the inputs $\xx$ be bounded, namely that they are distributed on the unit ball, $\|\xx\|=1$. More recent work from \citet{zhang2016a} suggests an objective which simultaneously maximizes the mean while minimizing the variance of the empirical margin distribution. Their routines are computationally tractable, but hyperparameter settings are non-trivial, and their risk bounds (in expectation) depend on the expected outcome of a leave-one-out cross-validation procedure, which is not characterized using interpretable quantities, reducing the utility of the bounds.

Another natural algorithmic strategy is to construct loss functions using more ``robust'' estimators of the \textit{true} expected margin $\exx y \, h(\xx)$, or related quantities such as the expected hinge loss $\exx \max\{1-y\,h(\xx),0\}$. In this regard the work of \citet{brownlees2015a} is highly relevant, in that sharp, descriptive risk bounds can be obtained for a wide class of learning algorithms, indeed any minimizer of such a loss. The practical downside is that computation is highly non-trivial and no procedures are proposed. The formal downside is that once again $\|\xx\|$ must be bounded for meaningful guarantees.

\paragraph{Our contributions}
To deal with the limitations of existing procedures highlighted above, the key idea here is to introduce a new convex loss that encourages the distribution of the margin to be tightly concentrated near a certain prescribed level. The procedure is easily implemented using gradient descent, admits formal performance guarantees reflecting both computational cost and optimization error, and aside from the usual cost of gradient computation there is virtually no computational overhead. Two key highlights are:
\begin{itemize}
\item The proposed algorithm enjoys high-probability risk bounds under moment bounds on $\xx$, and does not require $\|\xx\|$ to be bounded.

\item Numerical experiments show how a simple data-dependent re-scaling procedure can reduce the need for trial-and-error tuning of regularization.
\end{itemize}

\section{Algorithm introduction}

In this section we begin by introducing relevant algorithms from the literature, after which we introduce our proposed procedure.

\subsection{Related work}\label{sec:related_work}

Here we review the technical literature closely related to our work. Starting with the proposal of \citet{garg2003a}, their main theoretical results are a bound on the misclassification risk $\risk(h) \defeq \prr\{ y\,h(\xx) < 0 \}$ of $h(\xx) = \langle \ww, \xx \rangle + b$ for any $\ww \in \RR^{d}$ and $b \in \RR$. Assuming that $\|\xx\| = 1$, and given $2n$ observations, with probability no less than $1-4\delta$, we have
\begin{align}
\risk(h) \leq \widehat{\risk}(h) + \min_{d} \left( \mu_{d}(h) + 2 \sqrt{\frac{(d+2)\log(ne/(d+2))+\log(2\delta^{-1})}{2n}} \right)
\end{align}
where $\widehat{\risk}(h) = n^{-1}\sum_{i=1}^{n}I\{ y_{i}\,h(\xx_{i}) < 0 \}$, and the $\mu_{d}(h)$ term takes the form
\begin{align*}
\mu_{d}(h) \defeq \frac{2\delta^{-1}}{n} \sum_{i=1}^{2n} \min\left\{1, 3\exp\left(\frac{-h(\xx_{i})^{2}d}{2(2+|h(\xx_{i})|)^{2}}\right), \frac{2}{h(\xx_{i})^{2}d} \right\}.
\end{align*}
The projection error terms are derived from the fact that
\begin{align*}
\prr\{ h(\xx)\,\widetilde{h}(\widetilde{\xx}) < 0 \} \leq \min\left\{1, 3\exp\left(\frac{-h(\xx)^{2}d}{2(2+|h(\xx)|)^{2}}\right), \frac{2}{h(\xx)^{2}d} \right\}
\end{align*}
where $\widetilde{h}(\widetilde{\xx})=\langle P\ww, P\xx \rangle + b$, and $P$ is a $k \times d$ random matrix of independent Gaussian random variables, $N(0,1/d)$. Probability here is over the random draw of the matrix elements. Based on these guarantees, they construct a new loss, defined by
\begin{align*}
l(h;\zz) = \sum_{i \in \II_{+}} \exp\left(-\alpha h(\xx_{i})^{2}\right) + \sum_{i \in \II_{-}} \exp\left(-\beta y_{i}\,h(\xx_{i})\right),
\end{align*}
where $\II_{+}$ and $\II_{-}$ are respectively the indices of correctly and incorrectly classified observations. For correctly classified examples, they seek to minimize the projection error bound, whereas for incorrectly classified examples, then use a standard exponential surrogate loss. Depending on what $k \leq d$ minimizes their upper bound, the dependence on the number of parameters may be better than $O(\sqrt{d})$, but a price is paid in the form of $O(1/\delta)$ dependence on the confidence. On the computational side, proper settings of $\alpha$ and $\beta$ in practice is non-trivial.

The work of \citet{zhang2016a} considers using first- and second-order moments of the margin distribution as relevant quantities to build an objective. Writing
\begin{align*}
\overbar{m}(h) & \defeq \frac{1}{n} \sum_{i=1}^{n} y_{i} \, h(\xx_{i})\\
\overbar{v}(h) & \defeq \frac{1}{n} \sum_{i=1}^{n} \left(y_{i} \, h(\xx_{i}) - \overbar{m}(h) \right)^{2},
\end{align*}
in the case of $h(\xx) = \langle \ww,\xx \rangle$, they construct a loss
\begin{align*}
l(h;\zz) = \frac{\|\ww\|^{2}}{2} + \lambda_{1} \overbar{v}(h) - \lambda_{2} \overbar{m}(h) + \frac{\lambda_{3}}{n} \sum_{i=1}^{n} \max\{1-y_{i}\,h(\xx_{i}),0\},
\end{align*}
where the $\lambda_{1},\lambda_{2},\lambda_{3}$ are parameters to be set manually. The authors show how the optimization can be readily cast into an $n$-dimensional dual program of the form
\begin{align*}
\min_{\mv{\alpha} \in \RR^{n}} & \enspace \frac{1}{2} \mv{\alpha}^{T}U\mv{\alpha} + \uu^{T}\mv{\alpha}\\
\text{ s.t. } & \enspace 0 \leq \alpha_{i} \leq a_{i}, \quad i = 1,\ldots,n 
\end{align*}
for appropriate data-dependent matrix $U$, vector $\uu$, and weight bounds $a_{i}$, and they give some examples of practical implementations using dual coordinate descent and variance-reduced stochastic gradient descent. In all cases, parameter settings are left up to the user. Furthermore, statistical guarantees leave something to be desired; the authors prove that for any $\widehat{\mv{\alpha}}$ satisfying their dual objective, risk bounds hold as
\begin{align*}
\exx \risk(\widehat{\mv{\alpha}}) \leq \frac{1}{n} \exx\left(\sum_{i \in \II_{1}}\widehat{\alpha}_{i}U_{i,i} + |\II_{2}| \right),
\end{align*}
where expectation is taken with respect to the sample, $U_{i,i}$ are the diagonal elements of $U$, and the index sets are defined
\begin{align*}
\II_{1} & = \{i: 0 < \widehat{\alpha}_{i} < \lambda_{3}/n \}\\
\II_{2} & = \{i: \widehat{\alpha}_{i} = \lambda_{3}/n \}.
\end{align*}
These bounds provide limited insight into how and when the algorithm performs well, and in practice the algorithm requires substantial effort for model selection.

Finally, we consider the path-breaking analysis of \citet{brownlees2015a}, which greatly extends foundational work done by \citet{catoni2012a}. Letting $\varphi(u) = \max\{1-u,0\}$ denote the hinge loss, the Catoni estimator of the true location of a margin-based loss at candidate $h$, namely $\exx \varphi(y\,h(\xx))$, is defined as
\begin{align}\label{eqn:est_BJL}
\text{ any } \est(h) \geq 0 \enspace \text{ s.t. } \sum_{i=1}^{n} \psi\left(\frac{\est(h)-\varphi(y_{i}\,h(\xx_{i}))}{s}\right) = 0
\end{align}
where $s>0$ is a scaling parameter, and $\psi$ is a soft truncation function (see Figure \ref{fig:rho_psi_deriv}) defined by
\begin{align}\label{eqn:influence_cat17}
\psi(u) \defeq
\begin{cases}
u - u^{3}/6, & -\sqrt{2} \leq u \leq \sqrt{2}\\
2\sqrt{2}/3, & u > \sqrt{2}\\
-2\sqrt{2}/3, & u < -\sqrt{2}.
\end{cases}
\end{align}
The general analysis of \citet{brownlees2015a} provides a rich set of tools for obtaining risk bounds for any minimizer of this new robust objective function, namely bounds on $\risk(\widehat{h})$ where $\widehat{h}$ satisfies
\begin{align*}
\widehat{h} \in \argmin_{h \in \HH} \est(h),
\end{align*}
and $\HH$ denotes the hypothesis space our candidate lives in. Note that the $1$-Lipschitz continuity of the hinge loss gives us that for any candidates $g$ and $h$,
\begin{align*}
|\varphi(y\,g(\xx))-\varphi(y\,h(\xx))| \leq |y||g(\xx)-h(\xx)| = |g(\xx)-h(\xx)|,
\end{align*}
which means we can bound distances defined on the space $\{f(\xx)=\varphi(y\,h(\xx)): h \in \HH\}$ by distances on the space $\HH$. Going back to the linear model case of $h(\xx) = \langle \ww, \xx \rangle$, bounds in the $\LL_{2}$ distance $d_{2}$ can be constructed using
\begin{align*}
\exx |\varphi(y\,g(\xx))-\varphi(y\,h(\xx))|^{2} \leq \exx |g(\xx)-h(\xx)|^{2} \leq \|\ww_{g} - \ww_{h}\|^{2} \exx\|\xx\|^{2},
\end{align*}
and bounds in the $\LL_{\infty}$ distance take the form
\begin{align*}
\sup_{\xx} |\varphi(y\,g(\xx))-\varphi(y\,h(\xx))| \leq \sup_{\xx} |g(\xx)-h(\xx)| \leq \|\ww_{g} - \ww_{h}\| \sup_{\xx}\|\xx\|.
\end{align*}
Now, using their results, for large enough $s$ and $n$, one can show that with probability no less than $1-\delta$, it holds that
\begin{align*}
\exx \varphi\left(y\,\widehat{h}(\xx)\right) - \inf_{h \in \HH} \exx \varphi(y\,h(\xx)) \leq O\left( \sqrt{\frac{\log(3\delta^{-1})}{n}} + \log(2\delta^{-1})\left(\frac{\eta_{2}(\HH)}{\sqrt{n}} + \frac{\eta_{\infty}(\HH)}{n} \right) \right),
\end{align*}
where $c$ is a universal constant, and $\eta_{2}(\HH)$ and $\eta_{\infty}(\HH)$ are complexity terms. When these terms can be bounded, we can use the fact that the hinge loss is ``classification calibrated,'' and using standard results from \citet{bartlett2006b}, can obtain bounds on the excess misclassification risk based on the above inequality. The problem naturally is how to control these complexity terms. Skipping over some technical details, these terms can be bounded using covering number integrals dependent on $\HH$. As a concrete example, we have
\begin{align*}
\eta_{\infty}(\HH) \leq c_{\infty} \int_{0}^{\diameter(\HH;d_{\infty})} \log N(\epsilon,\HH,d_{\infty}) \, d\epsilon,
\end{align*}
where $d_{\infty}(g,h) = \sup_{\xx}|g(\xx)-h(\xx)|$ is the $\LL_{\infty}$ metric on $\HH$, the covering number $N(\epsilon,\HH,d_{\infty})$ is the number of $\epsilon$-balls in the $d_{\infty}$ metric needed to cover $\HH$, and $\diameter(\HH;d_{\infty}) = \sup\{d_{\infty}(g,h): g,h \in \HH\}$. In the case of $h(\xx) = \langle \ww, \xx \rangle$, this means $\|\xx\|$ must be almost surely bounded in order for the $\LL_{\infty}$ distance to be finite and the upper bounds to be meaningful. Under such assumptions, say $\ww$ comes from the unit ball and $\|\xx\| \leq B_{X}$ almost surely. Then ignoring non-dominant terms, the high-probability bounds can be specified as
\begin{align*}
\exx \varphi\left(y\,\widehat{h}(\xx)\right) - \inf_{h \in \HH} \exx \varphi(y\,h(\xx)) \leq O\left( \sqrt{\frac{\log(3\delta^{-1})}{n}} + \frac{\log(2\delta^{-1})d B_{X}}{\sqrt{n}} \right).
\end{align*}
While extremely flexible and applicable to a wide variety of learning tasks and algorithms, for the classification task, getting around the bound on $\xx$ is impossible using the machinery of \citet{brownlees2015a}. Even more serious complications are introduced by the difficulty of computation: while simple fixed-point procedures can be used to accurately approximate the robust objective $\est(h)$, it cannot be expressed explicitly, and indeed need not be convex as a function defined on $\HH$, even in the linear model case. Approximation error is unavoidable due to early stopping, and in addition to this computational overhead, using non-linear solvers to minimize the function $\est(h)$ can be costly and unstable in high-dimensional tasks \citep{holland2017a}. A recent pre-print from \citet{lecue2018a} considers replacing the M-estimator of \citet{brownlees2015a} with a median-of-means risk estimate, which does not require bounded inputs to get strong guarantees, but which requires an expensive iterative sub-routine for every loss evaluation, leading to substantial overhead for even relatively small learning tasks.

\subsection{Proposed algorithm}\label{sec:derivation}

We would like to utilize the strong elements of the existing procedures cited, while addressing their chief weaknesses. To do so, we begin by integrating the Catoni influence function $\psi$ defined in (\ref{eqn:influence_cat17}), which results in a new function of the form
\begin{align}\label{eqn:newloss_rho}
\rho(u) \defeq \begin{cases}
\frac{u^{2}}{2} - \frac{u^{4}}{24} & |u| \leq \sqrt{2},\\
|u| \frac{2\sqrt{2}}{3} - \frac{1}{2} & |u| > \sqrt{2}.
\end{cases}
\end{align}
Note that $\rho^{\prime}(u) = \psi(u)$ for all $u \in \RR$. This function satisfies $\rho(u) \geq 0$, is symmetric about zero so $\rho(u)=\rho(-u)$, and since the absolute value of the slope is bounded by $|\rho^{\prime}(u)| \leq 2\sqrt{2}/3$, we have that $\rho$ is Lipschitz continuous, namely that for any $u,v \in \RR$, we have $|\rho(u)-\rho(v)| \leq (2\sqrt{2}/3)|u-v|$.

\begin{figure}[h]
\centering
\includegraphics[width=0.75\textwidth]{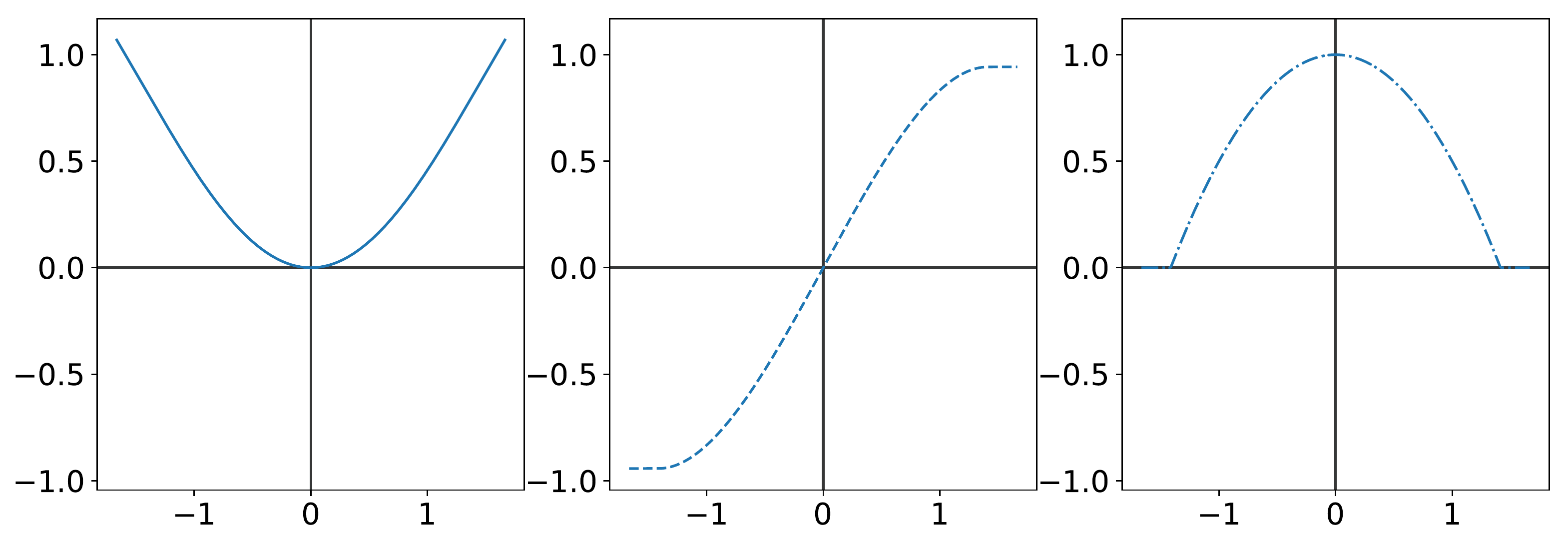}
\caption{Graphs of $\rho$, $\rho^{\prime}$ and $\rho^{\prime\prime}$ near the origin.}
\label{fig:rho_psi_deriv}
\end{figure}

Recalling the Catoni estimator (\ref{eqn:est_BJL}) used by \citet{brownlees2015a}, we define a new objective which is closely related:
\begin{align}\label{eqn:newloss_full}
\objnew(h;\gamma) \defeq \frac{s^{2}}{n} \sum_{i=1}^{n} \rho\left( \frac{\gamma - y_{i}\,h(\xx_{i})}{s} \right).
\end{align}
Here $\gamma \in \RR$ is the desired margin level, and once again $s>0$ is a re-scaling parameter. Note that this loss penalizes not only incorrectly classified examples, but also examples which are correctly classified, but \textit{overconfident}. The intuition here is that by also penalizing overconfident correct examples to some degree, we seek to constrain the variance of the margin distribution. The nature of this penalization is controlled by $\gamma$: a larger value leads to less correct examples being penalized.

It remains to set the scale $s$. To do so, first note that for any candidate $h$, we have
\begin{align*}
\est(h) \in \argmin_{\gamma} \objnew(h;\gamma),
\end{align*}
and that this estimator enjoys a pointwise error bound dependent on $s$ (see appendix \ref{sec:tech_proofs} for details), which says
\begin{align}\label{eqn:error_bound_s}
|\est(h) - \exx y\,h(\xx)| \leq \frac{\vaa y \, h(\xx)}{s} + \frac{2s\log(2\delta^{-1})}{n},
\end{align}
with probability no less than $1-\delta$. Minimizing this bound in $s>0$ naturally leads to setting $s^{2} = n \vaa y \, h(\xx) / 2\log(2\delta^{-1})$, but in our case, a certain amount of bias is assuredly tolerable; say a certain fraction $1/k$ of the desired $\gamma$ setting, plus error that vanishes as $n \to \infty$. By setting $s \geq \vaa y \, h(\xx) \, k / \gamma$ then, we have
\begin{align*}
|\est(h) - \exx y\,h(\xx)| \leq \frac{\gamma}{k} + O\left(\frac{1}{n}\right).
\end{align*}
The exact setting of $s>0$ plays an important role both in theory and in practice; we shall look at this in more detail in sections \ref{sec:theory}--\ref{sec:empirical}. In practice, the true variance will of course be unknown, but we can replace the true variance with any valid upper bound on the variance; rough estimates are easily constructed using moments of the empirical distribution (see section \ref{sec:empirical}).

With scaling taken care of, our proposed algorithm is simply to minimize the new loss (\ref{eqn:newloss_full}) using gradient descent, namely to run the iterative update
\begin{align*}
\hhat_{(t+1)} = \hhat_{(t)} - \alpha_{(t)} \nabla \objnew(\hhat_{(t)};\gamma),
\end{align*}
where $\alpha_{(t)}$ are step sizes. We summarize the key computations in Algorithm \ref{algo:mainGD} for the case of a linear model $h(\xx) = \langle \ww, \xx \rangle$ with fixed step sizes.

\begin{algorithm}
\caption{Margin pursuit by steepest descent.}
\label{algo:mainGD}
\begin{algorithmic}

\State \textbf{input:} $(\xx_{1},y_{1}),\ldots,(\xx_{n},y_{n}) \in \RR^{d} \times \{-1,1\}$

\smallskip

\State \textbf{parameters:} $\wwhat_{(0)} \in \RR^{d}$, $\gamma \in \RR$, $k > 0$, $\alpha > 0$

\smallskip

\State \textbf{scaling:} $\displaystyle s \geq \vaa y \, h(\xx) \, k / \gamma$

\smallskip

\For{$t = 0, 1, \ldots, T-1$} 
  
  \smallskip
  
  \State $\displaystyle \wwhat_{(t+1)} \gets \wwhat_{(t)} + \frac{s\,\alpha}{n} \sum_{i=1}^{n} \psi\left( \frac{\gamma - y_{i} \, \langle \wwhat_{(t)}, \xx_{i} \rangle}{s} \right) y_{i}\xx_{i}$
  
  \smallskip
\EndFor
\end{algorithmic}
\end{algorithm}

\begin{rmk}[Algorithm \ref{algo:mainGD} and distribution control]\label{rmk:unsatisfactory_learner}
Intuitively, in running Algorithm \ref{algo:mainGD} (or any generalization of it), the expectation is that with enough iterations, the approximation $\est(\wwhat_{(t)}) \approx \gamma$ should be rather sharp, although arbitrary precision assuredly cannot be guaranteed. If the $\gamma$ level is set too high given a hypothesis class $\HH$ with low complexity, we cannot expect $\gamma$ to be near the location of the margin $y \, h(\xx)$, which is accurately approximated by $\est(h)$. This can be easily proven: there exists a set of classifiers $\HH$ and distribution $\mu$ under which even a perfect optimizer of the new risk has a Catoni-type estimate smaller than $\gamma$ (proof given in appendix \ref{sec:tech_proofs}).
\end{rmk}

If the approximation $\est(\wwhat_{(t)}) \approx \gamma$ actually is sharp, how does this relate to \textit{control} of the margin distribution? By design, the estimator $\est(\cdot)$ is resistant to errant observations and is located near the \textit{majority} of observations (see Proposition \ref{prop:empirical_gamma_control}), if it turns out that $\est(\wwhat_{(t)})$ is close to $\gamma$, then it is \textit{not possible} for the majority of margin points be much smaller (or much larger) than $\gamma$.\footnote{Note that we still cannot rule out the possibility that the margin distribution is spread out over a wide region; a simple example is the case where the margins are symmetrically distributed around $\gamma$.}  Conceptually, the desired outcome is similar to that of the procedure of \citet{brownlees2015a} discussed in section \ref{sec:related_work}, but with an easy implementation and more straightforward statistical analysis. In section \ref{sec:theory}, we show that risk bounds are readily available for the proposed procedure, even without a bound on the inputs $\xx$. Empirical analysis in section \ref{sec:empirical} illustrates the basic mechanisms underlying the algorithm, using real-world benchmark data sets.

\section{Theoretical analysis}\label{sec:theory}

\paragraph{Notation}

For positive integer $k$, write the set of all positive integers no greater than $k$ by $[k] \defeq \{1,\ldots,k\}$. The underlying distribution of interest is that of $(\xx,y)$, here taking values on $\RR^{d} \times \{-1,1\}$. The data sample refers to $n$ independent and identically distributed (``iid'') copies of $(\xx,y)$, denoted $(\xx_{i},y_{i})$ for $i \in [n]$. Let $\HH$ denote a generic class of functions $h:\RR^{d} \to \RR$. The running assumption will be that all $h \in \HH$ are measurable, and at the very least satisfy $\exx |h(\xx)|^{2} < \infty$. Denote the input variance by $v_{X} \defeq \exx\|\xx\|^{2}$.

\paragraph{Scaling and location estimates}

Our chief interest from a theoretical standpoint is in statistical properties of Algorithm \ref{algo:mainGD}, in particular we seek high-probability upper bounds on the excess risk of the procedure after $T$ iterations, given $n$ observations, that depend on $T$, $n$, and low-order moments of the underlying distribution. We begin with some statistical properties of the motivating estimator, and a look at how scale settings impact these properties.

\begin{prop}[Scaling and location estimates]\label{prop:empirical_gamma_control}
For any $h \in \HH$ and scale $s>0$, the estimate $\est(h)$ satisfies the following:
\begin{enumerate}
\item There exists $0 < s^{\prime} < \infty$ such that for all $0 < s \leq s^{\prime}$, we have $\est(h) = \med \{ y_{i} \, h(\xx_{i}) \}_{i \in [n]}$.

\item There exists a constant $c>0$ such that for all $s>0$,
\begin{align*}
\left| \est(h) - \frac{1}{n} \sum_{i=1}^{n} y_{i} \, h(\xx_{i})\right| \leq \frac{c}{s^{2}}.
\end{align*}
\end{enumerate}
\end{prop}

\begin{rmk}
The basic facts laid out in Proposition 1 illustrate how $s$ controls the ``bias'' of the Catoni estimator. A larger scale factor makes the estimator increasingly sensitive to errant data, and causes it to close in on the empirical mean. A sufficiently small value on the other hand causes the estimator to effectively ignore the distribution tails, closing in on the empirical median.
\end{rmk}

\begin{prop}[Scaling and stability]\label{prop:stability}
Given any dataset $\zz_{1},\ldots,\zz_{n}$ and candidate $h \in \HH$, construct $\est(h)$ as usual. Then consider a modified dataset $\zz^{\prime}_{1},\ldots,\zz^{\prime}_{n}$, which is identical to the original except for one point, subject to arbitrary perturbation. Let $\est^{\prime}(h)$ denote the estimator under the modified data set. Defining a sub-index as
\begin{align*}
\II \defeq \left\{ i \in [n]: |\obj(h)-y_{i}\,h(\xx_{i})| \leq s\sqrt{2}/2 \right\},
\end{align*}
it follows that whenever $n$ and $s$ are large enough that $|\II| \geq n/2 > 24$, we have
\begin{align*}
|\obj(h)-\obj^{\prime}(h)| \leq \frac{s}{\sqrt{n}}.
\end{align*}
\end{prop}

\begin{rmk}
The stability property highlighted in Proposition \ref{prop:stability} is appealing because the difference $\max \{|y_{i}\,h(\xx_{i})-y_{i}^{\prime}\,h(\xx_{i}^{\prime})| : i \in [n]\}$ could be arbitrarily large, while the estimator $\obj(h)$ in shifting to $\obj^{\prime}(h)$ remains close to the majority of the points, and cannot be drawn arbitrarily far away. For clarity, we have considered the case of just one modified point, but a brief glance at the proof (in the appendix) should demonstrate how analogous results can readily be obtained for the case of larger fractions of modified points.
\end{rmk}

\begin{lem}[Pointwise error bound]\label{lem:pointwise_accuracy}
Fixing any $h \in \HH$, consider the estimate $\est(h)$ defined in (\ref{eqn:est_BJL}), equivalently characterized as a minimizer of $\objnew(h;\gamma)$ in $\gamma$, with scaling parameter $s$ set such that $s^{2} = nv/2\log(2\delta^{-1})$, where $v$ is any upper bound $\vaa y\,h(\xx) \leq v < \infty$. It follows that
\begin{align*}
\prr\left\{ |\est(h) - \exx y\,h(\xx)| > \sqrt{\frac{2v\log(2\delta^{-1})}{n}} \right\} \leq \delta.
\end{align*}
\end{lem}

\begin{rmk}
The confidence interval in Lemma \ref{lem:pointwise_accuracy} is called pointwise because it holds for a pre-fixed $h \in \HH$, in contrast with uniform bounds that hold independent of the choice of $h$. When considering our Algorithm \ref{algo:mainGD}, the candidate $h$ will be data-dependent and thus random, meaning that pointwise bounds will have to be extended to cover all possible contingencies; see the proof of Theorem \ref{thm:riskbd} for details.
\end{rmk}

\paragraph{Classification-calibrated loss}

Proceeding with our analysis, the ultimate evaluation metric of interest here is the classification risk (expectation of the zero-one loss), denoted
\begin{align}\label{eqn:risk_01}
R(h) \defeq \prr\{ \sign(h(\xx)) \neq y \}, \quad R^{\ast} \defeq \inf_{h \in \HH} R(h).
\end{align}
Using empirical estimates of the zero-one loss is not conducive to efficient learning algorithms, and our Algorithm \ref{algo:mainGD} involves the minimization of a new loss $\objnew(\cdot;\gamma)$, defined in equation (\ref{eqn:newloss_full}). To ensure that good performance in this metric implies low classification risk, the first step is to ensure that the function is \emph{calibrated} for classification, in the sense of \citet{bartlett2006b}. To start, fixing any $\gamma > 0$, define $\varphi(u) \defeq s^{2} \, \rho((\gamma - u)/s)$. This furnishes the surrogate risk
\begin{align}\label{eqn:risk_surrogate}
R_{\varphi}(h) \defeq \exx \varphi\left(y \, h(\xx)\right), \quad R_{\varphi}^{\ast} \defeq \inf_{h \in \HH} R_{\varphi}(h).
\end{align}
The basic idea is that if this loss $\varphi$ is calibrated, then one can show that there exists a function $\Psi_{s,\gamma}$ depending on user-specified $\gamma$ and $s$ settings, which is non-decreasing on the positive real line and satisfies
\begin{align*}
\Psi_{s,\gamma}( R(h) - R^{\ast} ) \leq R_{\varphi}(h) - R_{\varphi}^{\ast}.
\end{align*}
Our loss function $\rho$ defined in \ref{eqn:newloss_rho} is congenial due to the fact that it is classification-calibrated, with a $\Psi$-transform $\Psi_{s,\gamma}(\cdot)$ that can be computed exactly, for arbitrary values of $\gamma > 0$ and $s > 0$. Details of this computation are not difficult, but are rather tedious, and thus we relegate them to appendix \ref{sec:tech_getroot}. Basic facts are summarized in the following lemma.
\begin{lem}\label{lem:Psi_transform}
The loss function $\varphi(u) \defeq s^{2} \, \rho((\gamma - u)/s)$ is classification calibrated such that for each $\gamma > 0$, the following statements hold.
\begin{enumerate}
\item $\Psi$-transform: there exists a function $\Psi_{s,\gamma}:[0,1] \to \RR_{+}$ for which $\Psi_{s,\gamma}(R(h)-R^{\ast}) \leq R_{\varphi}(h) - R_{\varphi}^{\ast}$, depending on $\rho$, $s$, $\gamma$, and a concave function $H_{s,\gamma}(\cdot)$ defined on $[0,1]$, specified in the proof (also see Figure \ref{fig:H_Psi_inverse}). This $\Psi$-transform function takes the form
\begin{align*}
\Psi_{s,\gamma}(u) =  s^{2}\rho(\gamma/s) - H_{s,\gamma}\left(\frac{1+u}{2}\right).
\end{align*}

\item Risk convergence: given a sequence $(\hhat_{n})$ of sample-dependent $\{\zz_{1},\ldots,\zz_{n}\} \mapsto \hhat_{n}$, we have that convergence in our surrogate is sufficient for convergence in the zero-one risk, namely
\begin{align*}
\left\{ \lim\limits_{n \to \infty} R_{\varphi}(\hhat_{n}) = R_{\varphi}^{\ast} \right\} \subseteq  \left\{ \lim\limits_{n \to \infty} R(\hhat_{n}) = R^{\ast} \right\}.
\end{align*}

\item Invertibility: $\Psi_{s,\gamma}(u)$ is invertible on $[0,1]$, and thus for small enough excess risk, we can bound as $R(h)-R^{\ast} \leq \Psi_{s,\gamma}^{-1}(R_{\varphi}(h) - R_{\varphi}^{\ast})$.
\end{enumerate}
\end{lem}

\begin{rmk}[Generalization and $\gamma$ level setting]
One would naturally expect that all else equal, if a classifier achieves the same excess $\varphi$-risk for a larger value of $\gamma$, then the resulting excess classification risk should be smaller, or at least no larger. More concretely, we should expect that
\begin{align*}
\gamma \leq \gamma^{\prime} \implies \Psi_{s,\gamma}^{-1}(a) \geq \Psi_{s,\gamma^{\prime}}^{-1}(a), \quad a \in [0,s^{2}\,\rho(\gamma/s)].
\end{align*}
This range comes from the fact that $\Psi_{s,\gamma}(0) = 0$ and $\Psi_{s,\gamma}(1) = s^{2} \, \rho(\gamma/s)$. This monotonicity follows from the definition of $\rho$ and the convexity of the $\Psi$-transform (also see Figure \ref{fig:H_Psi_inverse} in the following section).
\end{rmk}

\paragraph{Assumptions and risk bounds, with discussion}

With preparatory results in place, we can now pursue an excess risk bound for Algorithm \ref{algo:mainGD}. To make notation more transparent, we accordingly write $\risk(\ww)$ and $\rnew(\ww)$ to denote the respective risks under $\HH = \{h: h(\xx)=\langle \ww, \xx \rangle, \ww \in \WW \}$, where $\WW \subset \RR^{d}$. The core technical assumptions are as follows:
\begin{itemize}
\item[\namedlabel{asmp:model_compact}{A0}.] $\WW$ is a compact subset of $\RR^{d}$, with diameter $\Delta \defeq \sup\{\|\uu-\vv\|: \uu,\vv \in \WW\} < \infty$.

\item[\namedlabel{asmp:risk_flat_min}{A1}.] There exists $\wwstar \in \WW$ at which $\rnew^{\prime}(\wwstar)=0$.

\item[\namedlabel{asmp:risk_strong_convex}{A2}.] $\rnew(\ww)$ is $\kappa$-strongly convex on $\WW$, with minimum\footnote{Assuming we can take the derivative under the integral, the smoothness of $\rho$ implies differentiability of $R_{\varphi}$. Then using the compactness of $\WW$, it follows that $\wwstar \in \WW$.} denoted by $\wwstar$.

\item[\namedlabel{asmp:sub_gaussian}{A3}.] The gradient distribution follows a standard form of high-dimensional sub-Gaussianity, characterized as follows. Writing $\mv{b}(\ww) \defeq -\rho^{\prime}(\gamma - y \langle \ww, \xx \rangle ) y \, \xx$ for the new loss gradient before scaling by $s$, and $\Sigma(\ww)$ for its covariance matrix, there exists some $c>0$ such that for all $\ww \in \WW$, $a \geq 0$, and $\|\uu\|=1$, we have
\begin{align*}
\exx\exp\left( a \langle \uu, \mv{b}(\ww) - \exx\mv{b}(\ww) \rangle \right) \leq \exp\left(c a^{2} \langle \uu, \Sigma(\ww)\uu \rangle \right).
\end{align*}
\end{itemize}

\begin{rmk}[Feasibility of assumptions]
The important assumptions here are \ref{asmp:risk_strong_convex} and \ref{asmp:sub_gaussian}. The latter can be satisfied with inputs $\xx$ that have sub-Gaussian tails; this does not include data with higher-order moments that are infinite, but requires no bound on $\|\xx\|$ at all. As for the former assumption \ref{asmp:risk_strong_convex}, first note that the $(i,j)$th element of the Hessian of the new loss function is
\begin{align*}
\frac{\partial^{2}}{\partial w_{i} \partial w_{j}} s^{2} \, \rho\left(\frac{\gamma - y \, \langle \ww, \xx \rangle}{s}\right) = \rho^{\prime\prime}\left(\frac{\gamma - y \, \langle \ww, \xx \rangle}{s}\right) x_{i}x_{j}, \quad i,j \in [d]
\end{align*}
where
\begin{align*}
\rho^{\prime\prime}(u) =
\begin{cases}
1 - u^{2}/2, & \text{ if } |u| \leq \sqrt{2}\\
0, & \text{ else.}
\end{cases}
\end{align*}
Write $q = \uu^{T}(\xx\xx^{T})\uu$ for readability, and use $\exx_{+}$ and $\exx_{-}$ to denote integration over the positive and non-positive parts of $q$. First, observe that
\begin{align*}
\exx_{-} \rho^{\prime\prime}\left(\frac{\gamma - y \, \langle \ww, \xx \rangle}{s}\right) q & = \exx I\{q \leq 0\} \rho^{\prime\prime}\left(\frac{\gamma - y \, \langle \ww, \xx \rangle}{s}\right) q\\
& \geq \exx I\{q \leq 0\} q\\
& = \exx q - \exx_{+} q.
\end{align*}
Using this inequality, we have
\begin{align*}
\uu^{T} \rnew^{\prime\prime}(\ww) \uu & = \exx \rho^{\prime\prime}\left(\frac{\gamma - y \, \langle \ww, \xx \rangle}{s}\right) q\\
& = \exx_{+} \rho^{\prime\prime}\left(\frac{\gamma - y \, \langle \ww, \xx \rangle}{s}\right) q + \exx_{-} \rho^{\prime\prime}\left(\frac{\gamma - y \, \langle \ww, \xx \rangle}{s}\right) q\\
& \geq \exx_{+} \rho^{\prime\prime}\left(\frac{\gamma - y \, \langle \ww, \xx \rangle}{s}\right) q + \left( \exx q - \exx_{+} q \right)\\
& = \exx q + \exx_{+} \left(\rho^{\prime\prime}\left( \frac{\gamma - y \, \langle \ww, \xx \rangle}{s}\right) - 1 \right) q.
\end{align*}
The second term on the right-hand side is a negative value that can be taken near zero for any $\ww \in \WW$ by taking $s>0$ large enough. The first term is $\exx q = \uu^{T}\exx\xx\xx^{T}\uu$, and thus with large enough $s$, as long as the second moment matrix of the inputs is positive definite satisfying $\exx\xx\xx^{T} \succeq c I_{d}$ for some $c>0$ (a weak assumption), it follows that there exists a $\kappa>0$ such that $\rnew^{\prime\prime}(\ww) \succeq \kappa I_{d}$ holds. Since the risk is twice continuously differentiable, This implies $\kappa$-strong convexity \citep[Theorem 2.1.11]{nesterov2004ConvOpt}.
\end{rmk}

With these assumptions in place, finite-sample risk bounds can be obtained.

\begin{thm}\label{thm:riskbd}
Running Algorithm \ref{algo:mainGD} for $T$ iterations, the final output produced, written $\wwhat_{(T)}$, for constant $c>0$ and $\beta \defeq 2\kappa v_{X} / (\kappa + v_{X})$ satisfies
\begin{align*}
\risk(\wwhat_{(T)}) - \risk^{\ast} \leq  \Psi_{s,\gamma}^{-1}\left( (1-\alpha\beta)^{T} v_{X} \|\wwhat_{(0)}-\wwstar\|^{2} + \frac{4v_{X}}{\beta^{2}n}\left((1+\delta)v_{X} + 2s\,\varepsilon^{\ast}\right)^{2} \right)
\end{align*}
with probability no less than $1-2\delta$ over the random draw of the sample, where the dominant term $\varepsilon^{\ast}$ is defined
\begin{align*}
\varepsilon^{\ast} \defeq \sqrt{c\rho^{\prime}(\sqrt{2})^{2}\exx\|\xx\xx^{T}\|(d\log(3\sqrt{n}(2\delta)^{-1}) + \log(\delta^{-1}))}.
\end{align*}
\end{thm}

\begin{rmk}[Interpretation and tradeoffs]
Excess risk bounds give in Theorem \ref{thm:riskbd} are composed of two key terms, one of a computational nature, and one of a statistical nature. The first term is optimization error, which decreases as $T$ grows, and depends on the initial estimate $\wwhat_{(0)}$, the step-size $\alpha$, and the convexity of the surrogate risk through $\beta$. The second term is statistical error, and depends on the sample size, scale $s$, the number of parameters, and second-order moments of the inputs $\xx$. Note that there is a clear tradeoff due to $s$: a sufficiently large scale factor is needed to ensure \ref{asmp:risk_strong_convex} holds (yielding large enough $\beta$), but setting $s$ too large impacts the statistical error in a negative way.

Finally, we note the $d$ factor in $\varepsilon^{\ast}$ is due to a covering number argument used to obtain a bound on the empirical gradient error that holds uniformly over $\ww \in \WW$. Does there exist another computational procedure, with the \textit{same optimization error}, and without this seemingly superfluous $d$ factor in the statistical error? We pursue such analysis in future work.
\end{rmk}

\section{Empirical analysis}\label{sec:empirical}

In our numerical experiments, we aim to complement the theoretical analysis carried out in the previous section. We look at how algorithm parameter settings impact generalization guarantees, and using real-world datasets, investigate how Algorithm \ref{algo:mainGD} performs, comparing its behavior with a benchmark procedure.

\paragraph{Margin level, scale, and generalization}

First, we look at the function $\Psi_{s,\gamma}$ introduced in the previous section, and its inverse, $\Psi_{s,\gamma}^{-1}$. In the two leftmost plots of Figure \ref{fig:H_Psi_inverse}, we plot the graph of $H_{s,\gamma}(\cdot)$ and $\Psi_{s,\gamma}(\cdot)$ over $[0,1]$, for $s=1$ and varying values of $\gamma$. Convexity of $\Psi_{s,\gamma}$ and its monotonic dependence on $\gamma$ can be clearly observed.

In the second plot from the right, we fix $a$ and $s=1$ and plot the graph of $\Psi_{s,\gamma}^{-1}(a)$ over a range of $\gamma$ and $a$ values. We can clearly observe how achieving a better excess surrogate risk (corresponding to a smaller $a$ value) for a larger $\gamma$ value leads to smaller excess misclassification risk (corresponding to smaller values of the function plotted). In addition, the same excess surrogate risk $a$ clearly leads to better generalization in the misclassification risk if it is achieved with a larger $\gamma$ value, although this positive impact diminishes quickly as $\gamma$ gets large.

Finally, in the rightmost plot of Figure \ref{fig:H_Psi_inverse}, we fix $\gamma$ and $\varepsilon = 0.5$, and plot $\Psi_{s,\gamma}^{-1}(\varepsilon/s)$ for a range of positive $s$ values. In the limit as $s$ gets large, we find that this quantity bottoms out quickly at a positive value. This has important implications in terms of scaling strategies, because it demonstrates where issues can arise with scaling $s \to \infty$ with $n \to \infty$, as would be implied by simply minimizing the pointwise error bound (as seen in (\ref{eqn:error_bound_s}) and Lemma \ref{lem:pointwise_accuracy}). Indeed, if any algorithm can achieve an excess surrogate risk of $O(n^{-1/2})$ (corresponding to $\varepsilon$), if $s$ is allowed to scale as $O(\sqrt{n})$, then even taking $n$ large will not imply a small misclassification risk. This is one important reason that Algorithm \ref{algo:mainGD} does not scale using the bound-minimizing $s$ value, but rather a value that allows for consistency in the limit as $n$ and $T$ grow large.

\begin{figure}[t]
\centering
\includegraphics[width=0.5\textwidth]{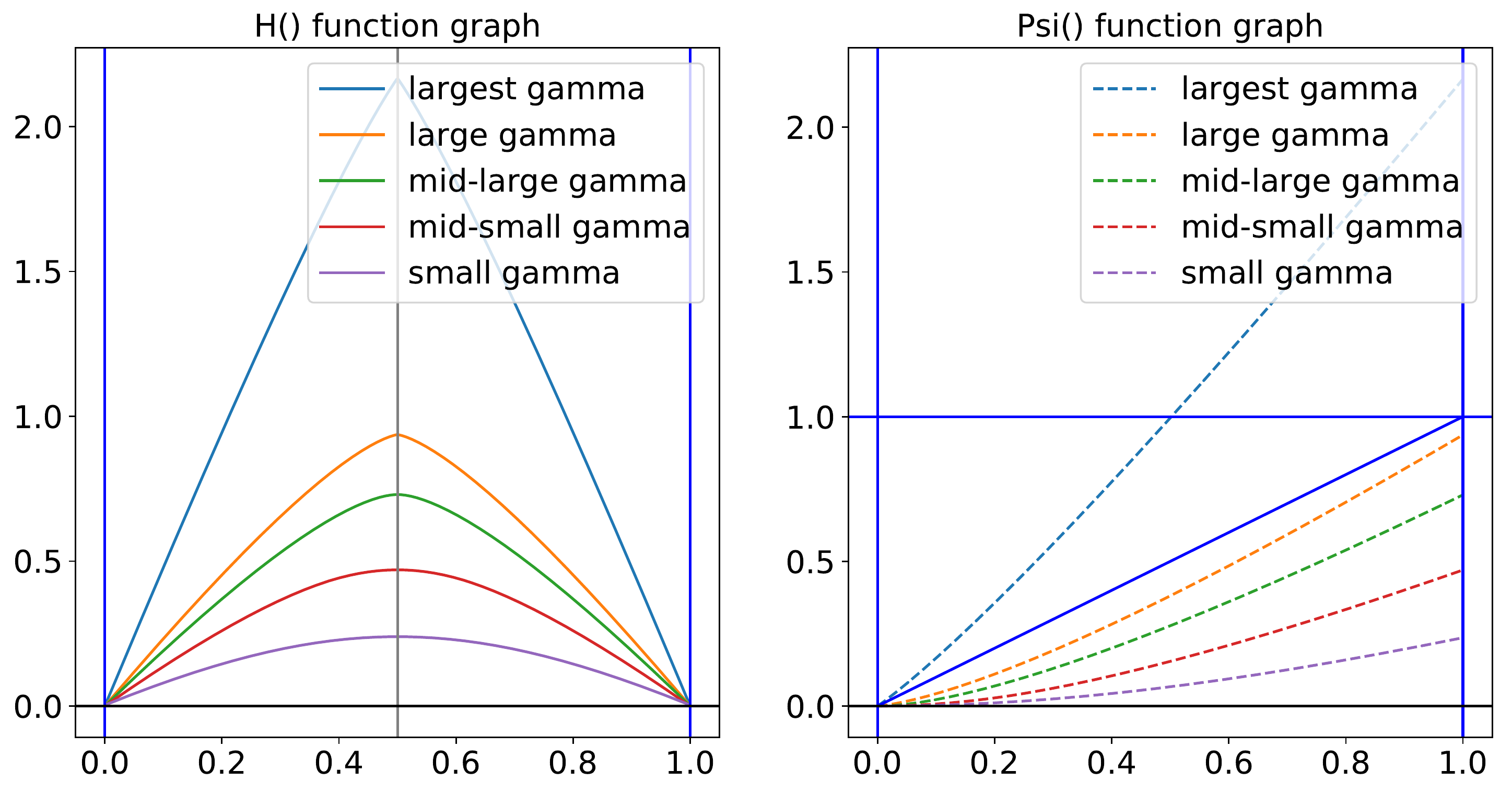}\includegraphics[width=0.25\textwidth]{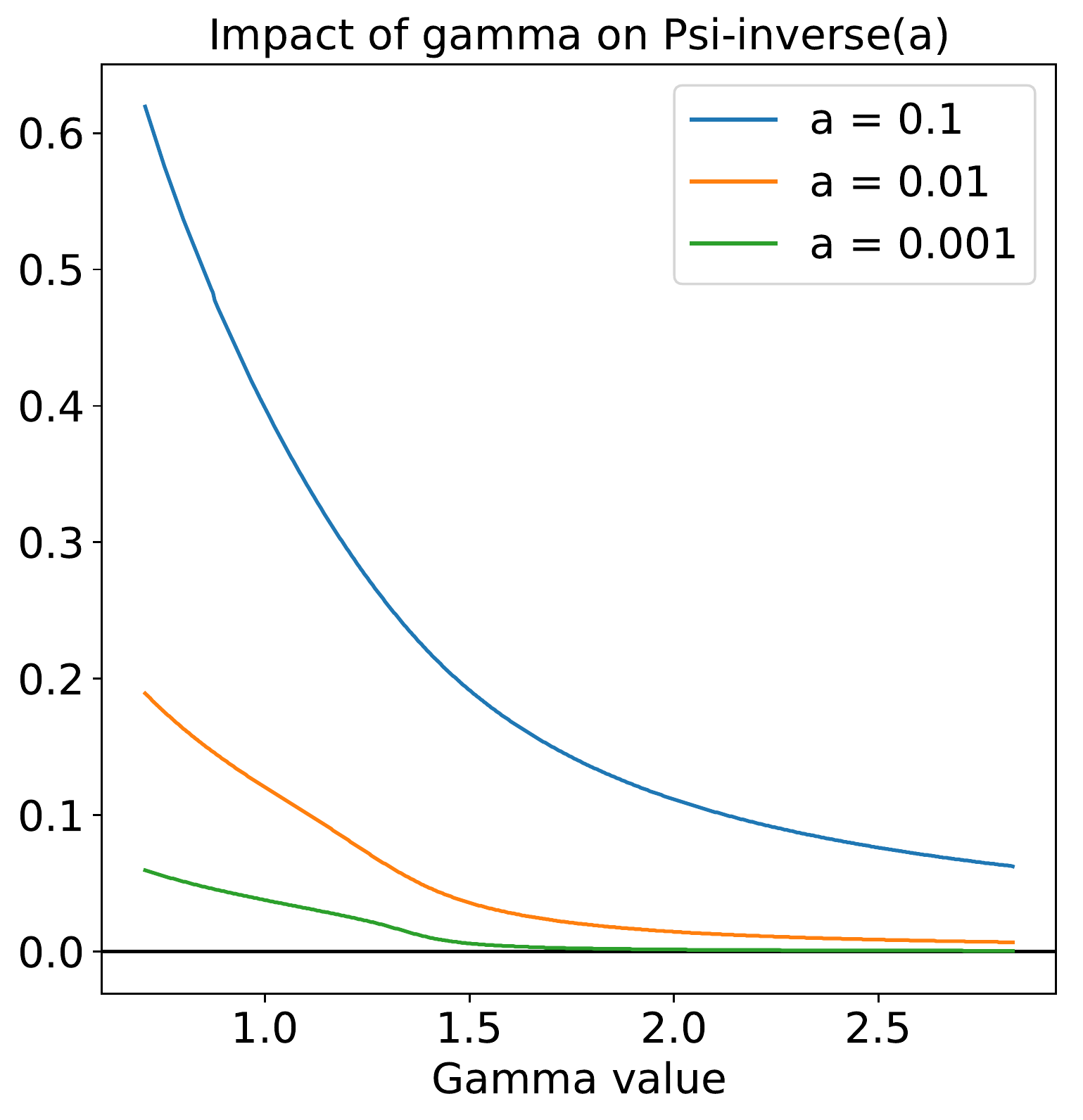}\includegraphics[width=0.25\textwidth]{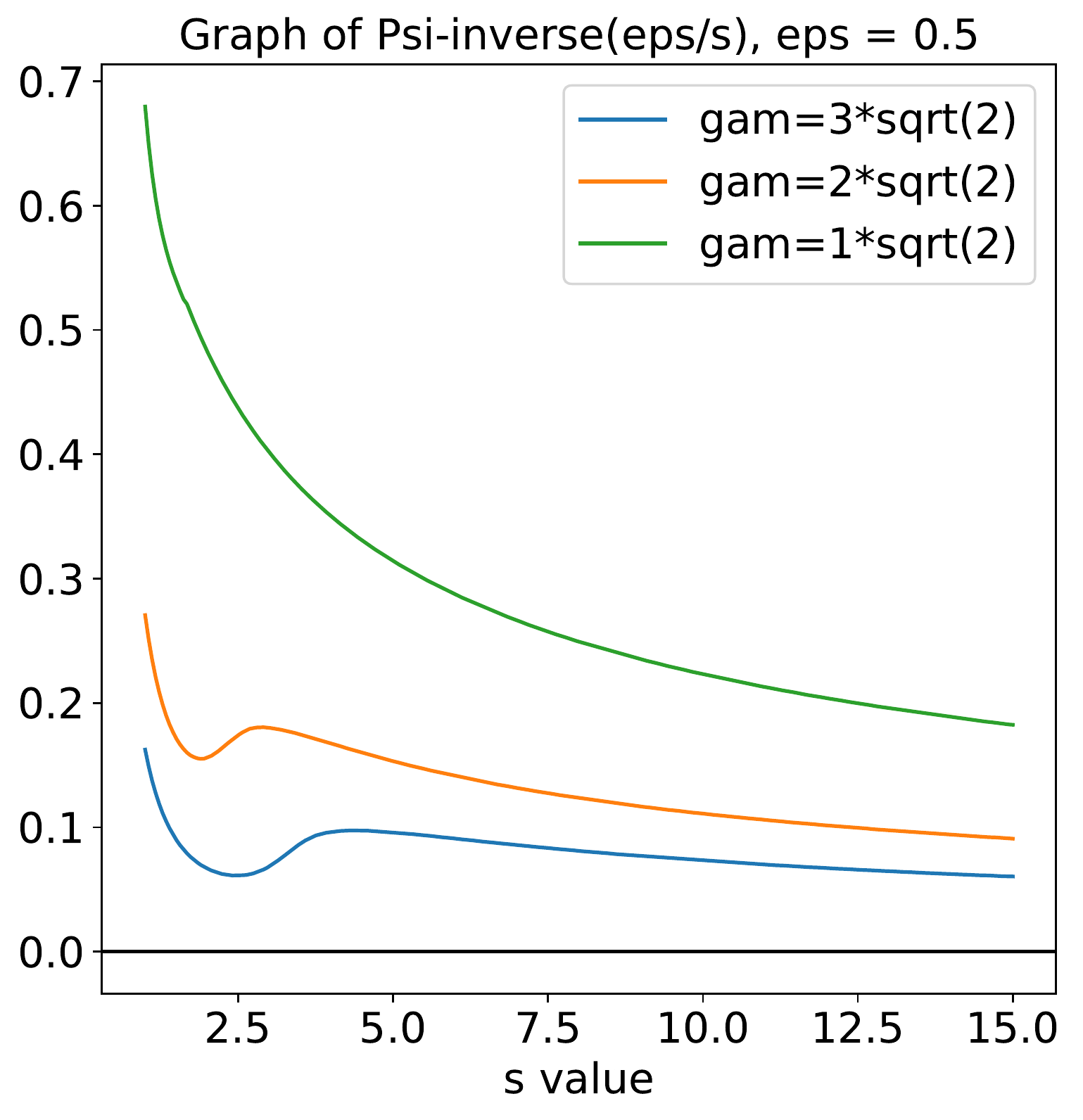}
\caption{Graphs of quantities related to the $\Psi$-transform of the proposed loss, namely $\Psi_{s,\gamma}$. In the leftmost two plots, from smallest to largest, the $\gamma$ values are $\gamma = \sqrt{2}/2, \sqrt{2}-0.4, \sqrt{2}-0.11, \sqrt{2}+0.11, 2\sqrt{2}$. Computation of the inverse is approximate, and done as follows. For any $(s,\gamma)$ pair, we compute $\Psi_{s,\gamma}(u)$ for $u \in [0,1]$ over a uniformly spaced grid $0 = u_{1} \leq u_{2} < \cdots < u_{K} = 1$, with $K=2500$. The approximate value is then given as $\Psi_{s,\gamma}^{-1}(a) = u_{k^{\ast}}$, where $k^{\ast} = \max\{k \in [K]: \Psi_{s,\gamma}(u_{k}) \leq a\}$.}
\label{fig:H_Psi_inverse}
\end{figure}

\paragraph{Benchmark data tests: experimental setup}

In all the experiments discussed here, we consider binary classification on real-world data sets, modified to control for unbalanced ratios of positive and negative labels. Training for each data set is done using pair $(\mv{X},\yy)$, where $\mv{X}$ is $n \times d$, and $\yy$ is $n \times 1$, and testing is done on a disjoint subset. The train-test sequence is repeated over 25 trials, and all numerical performance metrics displayed henceforth should be assumed to be averages taken over all trials.

We use four data sets, denoted \textsc{cov}, \textsc{digit5}, \textsc{protein}, and \textsc{sido}, creating subsets under the following constraints: (1) Sample size $n$ is no more than ten times the nominal dimension $d$, and (2) both the training and testing data sets have balanced ratios of labels (as close as possible to $50\%$ each). Starting with \textsc{cov} ($n=540$, $d=54$, non-zero: $22\%$), this is the ``Forest CoverType dataset'' on the UC Irvine repository, converted into a binary task identifying class 1 against the rest. \textsc{digit5} ($n=5000$, $d=784$, non-zero: $19\%$) is the MNIST hand-written digit data, converted into a binary task for the digit 5. \textsc{protein} ($n=740$, $d=74$, non-zero: $99\%$) is the protein homology dataset (KDD Cup 2004). \textsc{sido} ($n=425$, $d=4932$, non-zero: $11\%$) is the molecular descriptor data set (NIPS 2008 causality challenge), with binary-valued features. In each trial, from the full original data set, we take a random sub-sample of the specified size, without replacement, for training, and for test data we use as much of the remaining data as possible, within the confines of constraint (2) above.

As a well-known benchmark algorithm against which we can compare the behaviour and performance of the proposed Algorithm \ref{algo:mainGD}, we implement and run the well-known Pegasos algorithm of \citet{shalev2011a}. For both methods, the initial value $\wwhat_{(0)}$ is determined randomly in each trial. We explore multiple settings of Algorithm \ref{algo:mainGD} described further below, but in all cases we take the stochastic optimization approach: instead of using all $n$ training examples at each step, we randomly select one at a time for computing the update direction, and use a step size of . For direct comparison with Pegasos, we set the margin level to $\gamma = 1$, add a squared $\ell_{2}$-norm regularization term with coefficient $\lambda$, utilizing a step size of $\alpha = (s\sqrt{\lambda}(1+t))^{-1}$, and projecting to the $1/\sqrt{\lambda}$-radius ball. That is, we run a stochastic projected gradient descent version of Algorithm \ref{algo:mainGD}, and evaluate the impact of the proposed loss function.

\paragraph{Benchmark data tests: generalization with naive scaling}

We begin with the simplest setting of Algorithm \ref{algo:mainGD}, where $s=1$ is fixed throughout. In Figures \ref{fig:realtest_naive_scale_1}--\ref{fig:realtest_naive_scale_2}, we plot training error, test error, and numerous statistics of the empirical margin distribution, all as a function of cost incurred (equal to number of gradients computed). For each dataset, we experimented with $\lambda \in \{10^{0},10^{-6},10^{-6},\ldots,10^{-1}\}$ and display the results for the case of $\lambda$ that resulted in the best performance, as measured by the lowest test error achieved over all iterations.

We see that our proposed procedure is highly competitive with the best setting of Pegasos, and results in a margin distribution very distinct from that of the competing procedure. On the whole, we see a much more symmetrical distribution, with smaller variance, that over iterations pushes the margin location up in a monotonic fashion, in stark contrast to that of Pegasos, whose empirical distribution peaks early and slowly settles down over time. The smaller variance and higher degree of symmetry is precisely what we would expect given the definition of $\rho$, which assigns a penalty for correctly classified examples that are overconfidently classified, as discussed in section \ref{sec:derivation}.

\begin{figure}[h]
\centering
(\textsc{cov}) \hspace{6.5cm} (\textsc{digit5})
\smallskip
\includegraphics[width=0.5\textwidth]{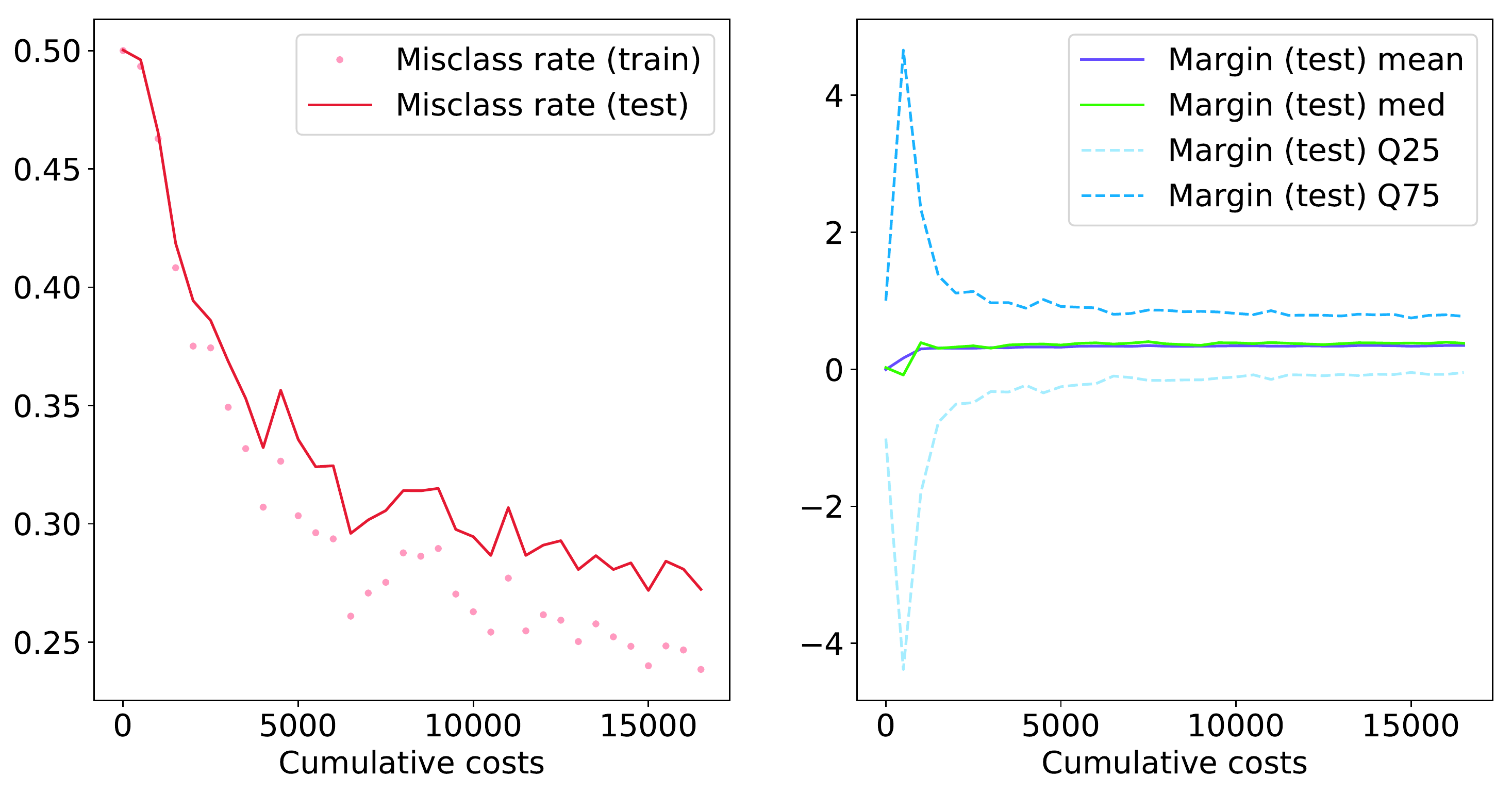}\includegraphics[width=0.5\textwidth]{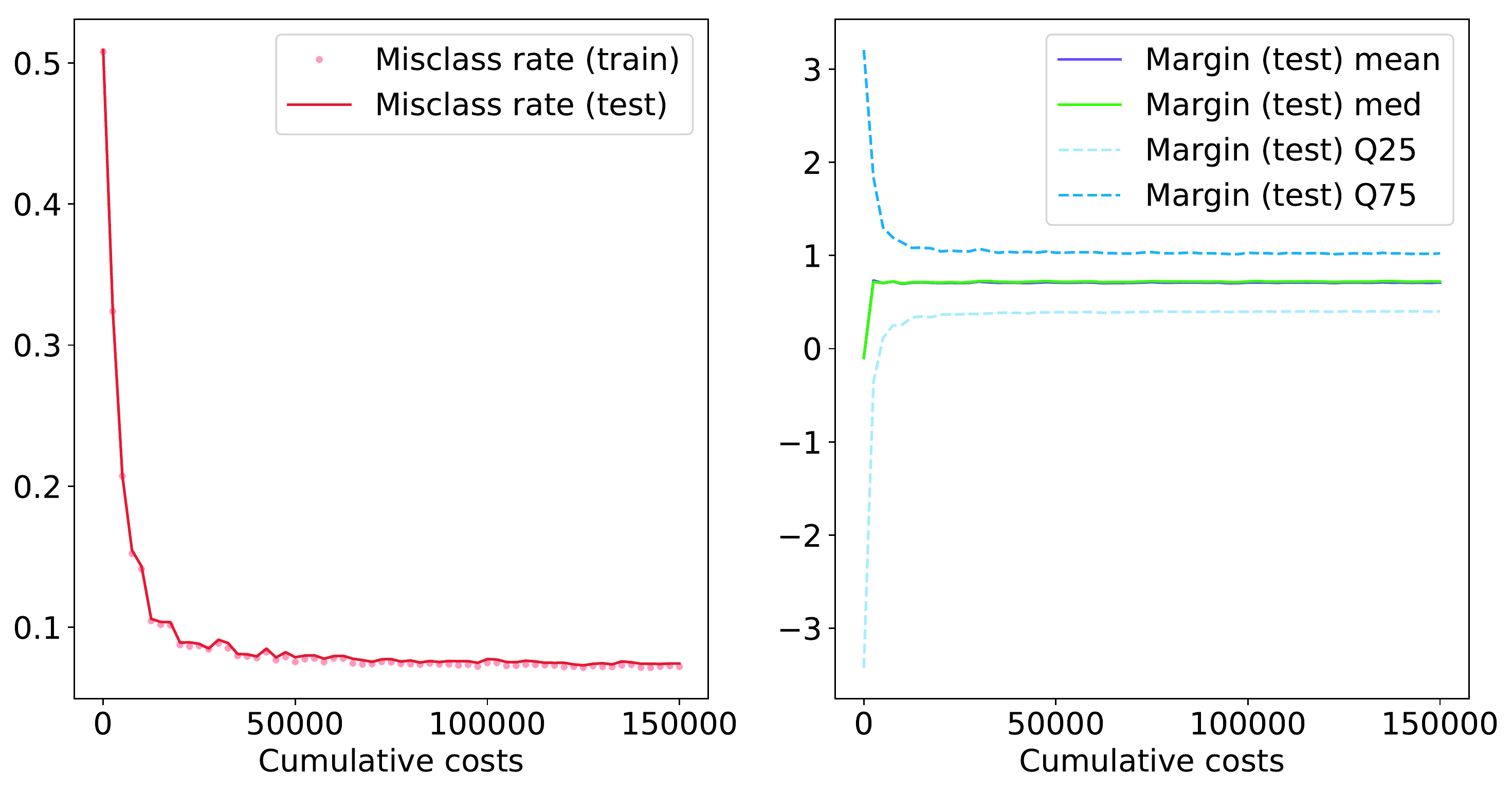}\\
\includegraphics[width=0.5\textwidth]{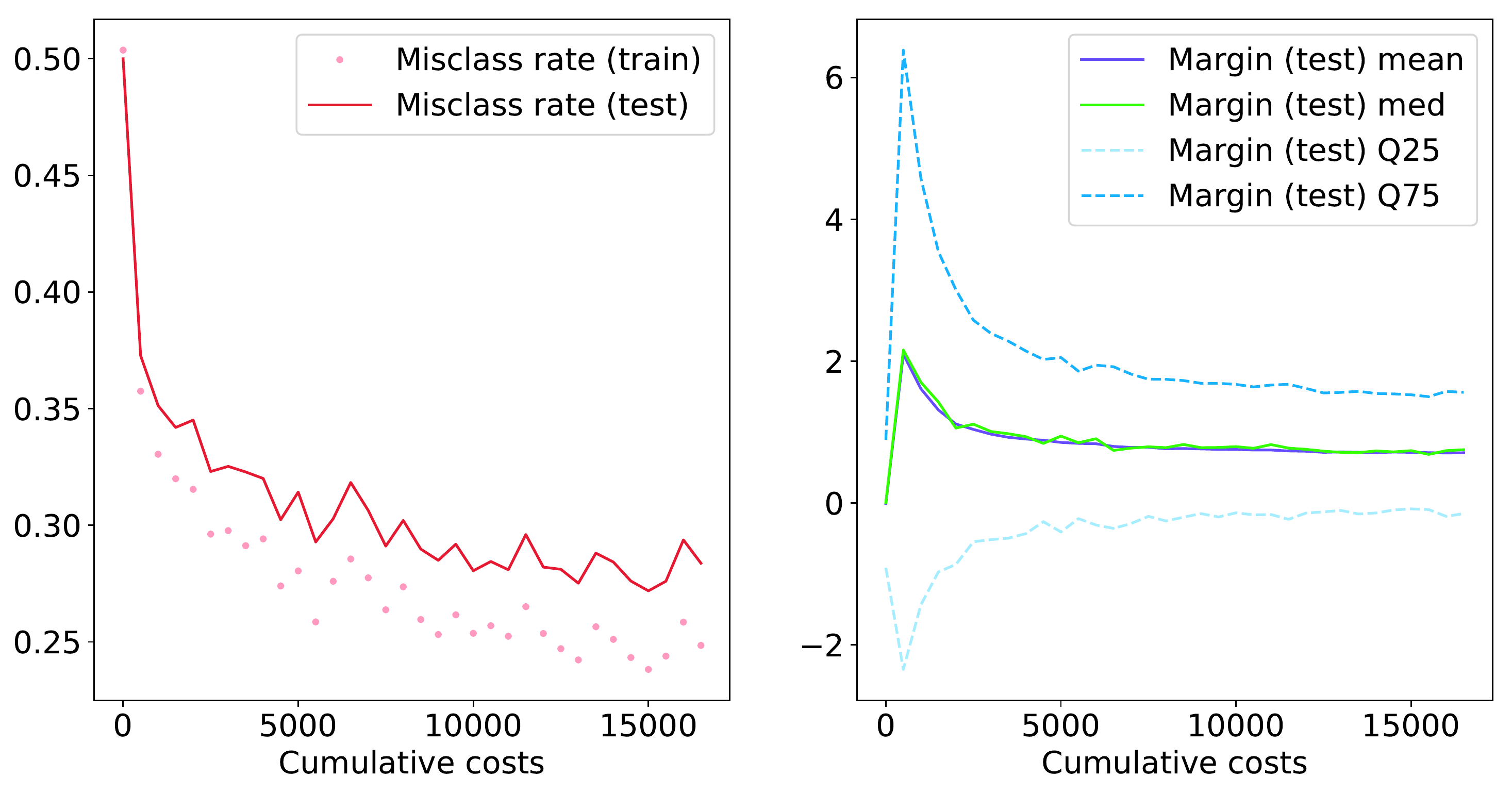}\includegraphics[width=0.5\textwidth]{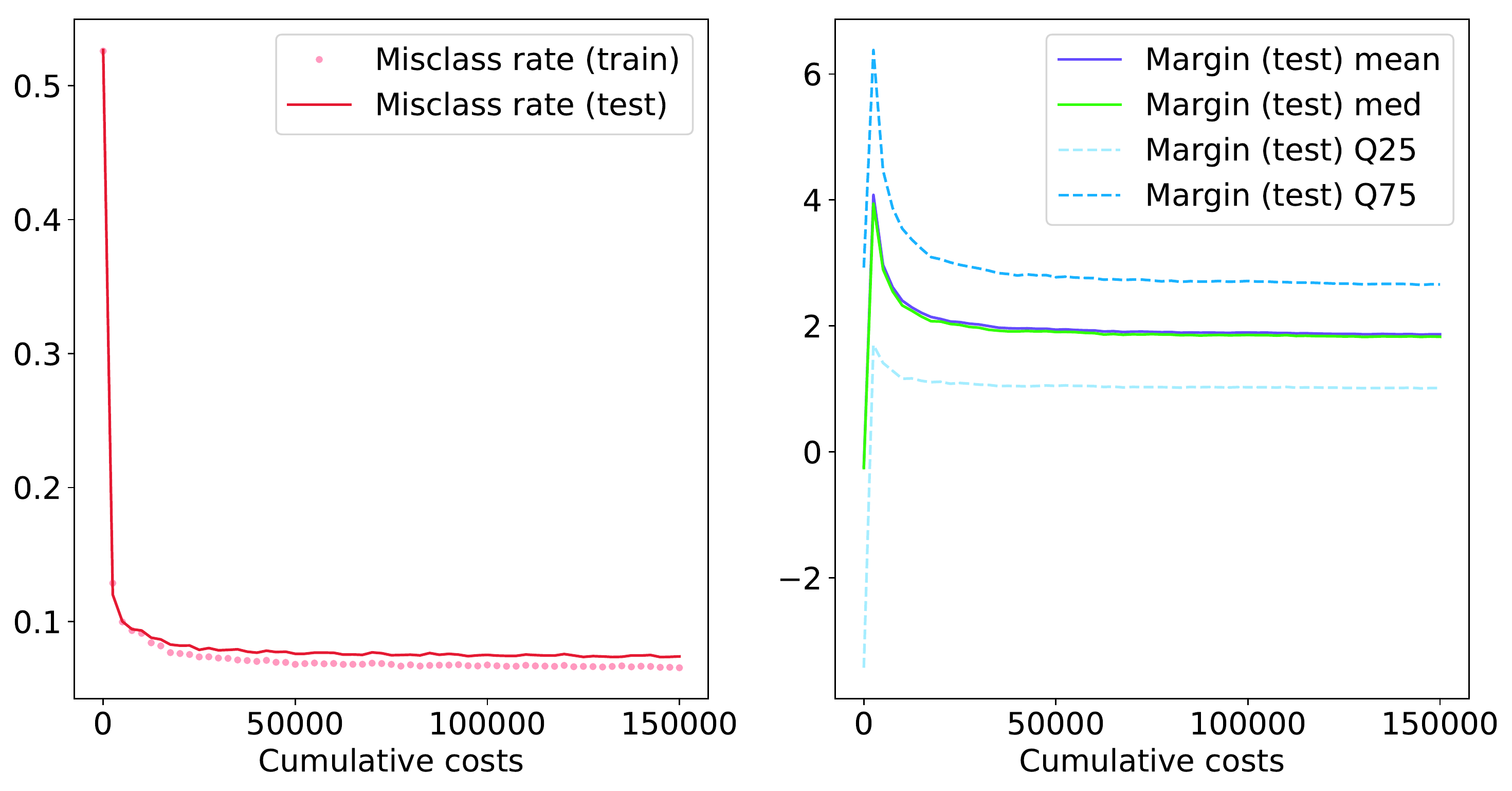}
\caption{Top row: Algorithm \ref{algo:mainGD}. Bottom row: Pegasos.}
\label{fig:realtest_naive_scale_1}
\end{figure}

\begin{figure}[h]
\centering
(\textsc{protein}) \hspace{6.5cm} (\textsc{sido})
\smallskip
\includegraphics[width=0.5\textwidth]{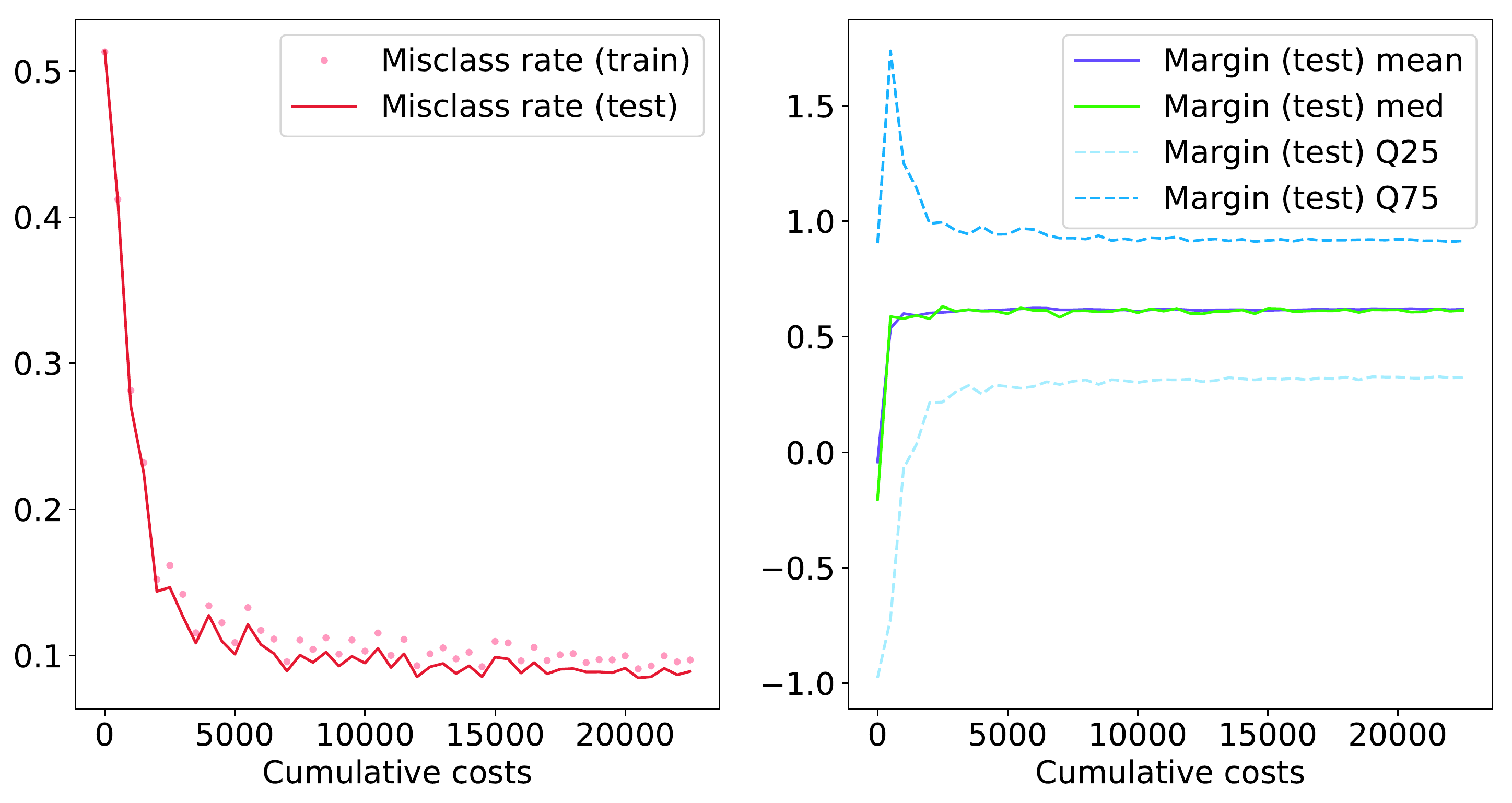}\includegraphics[width=0.5\textwidth]{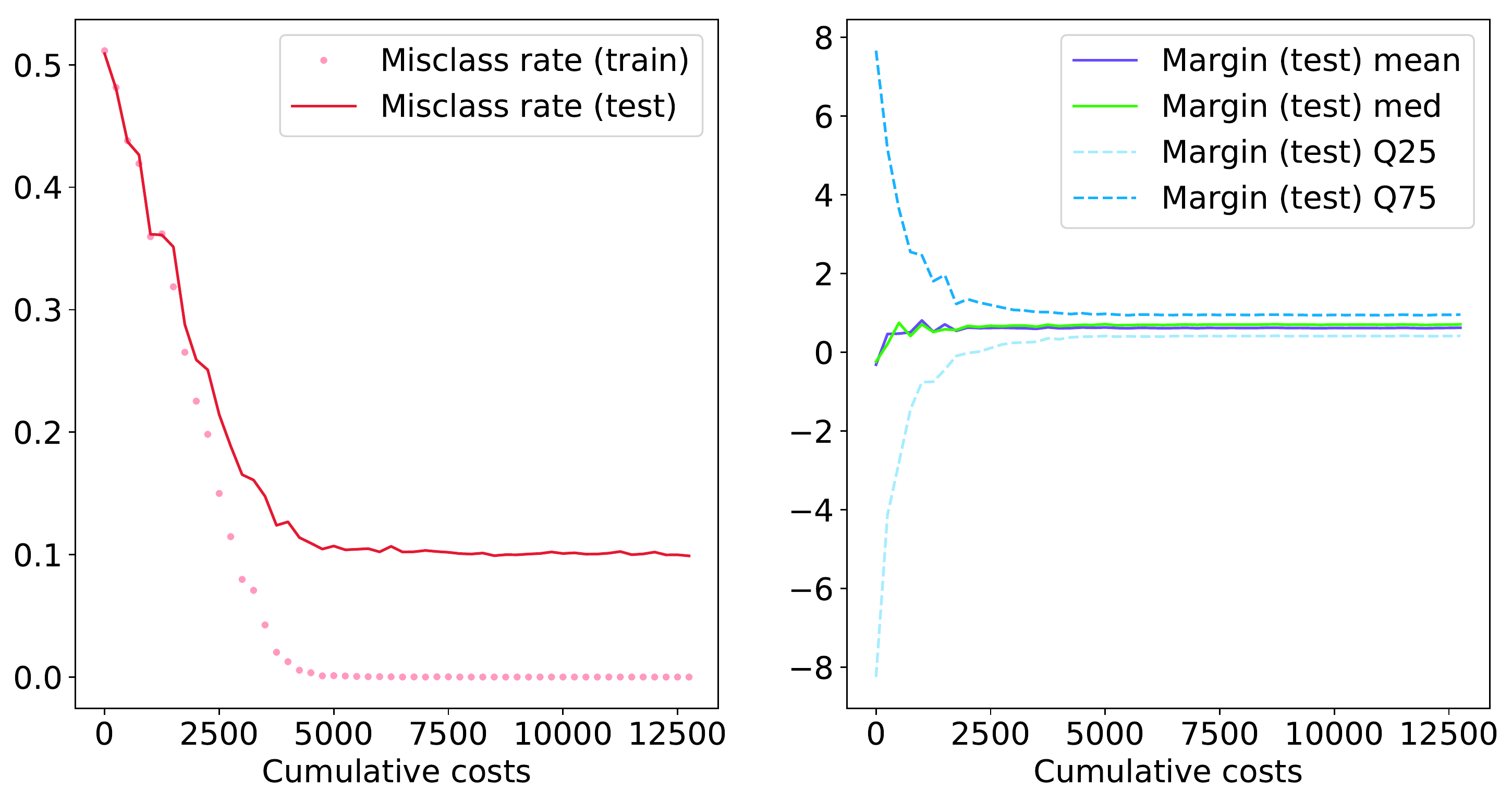}\\
\includegraphics[width=0.5\textwidth]{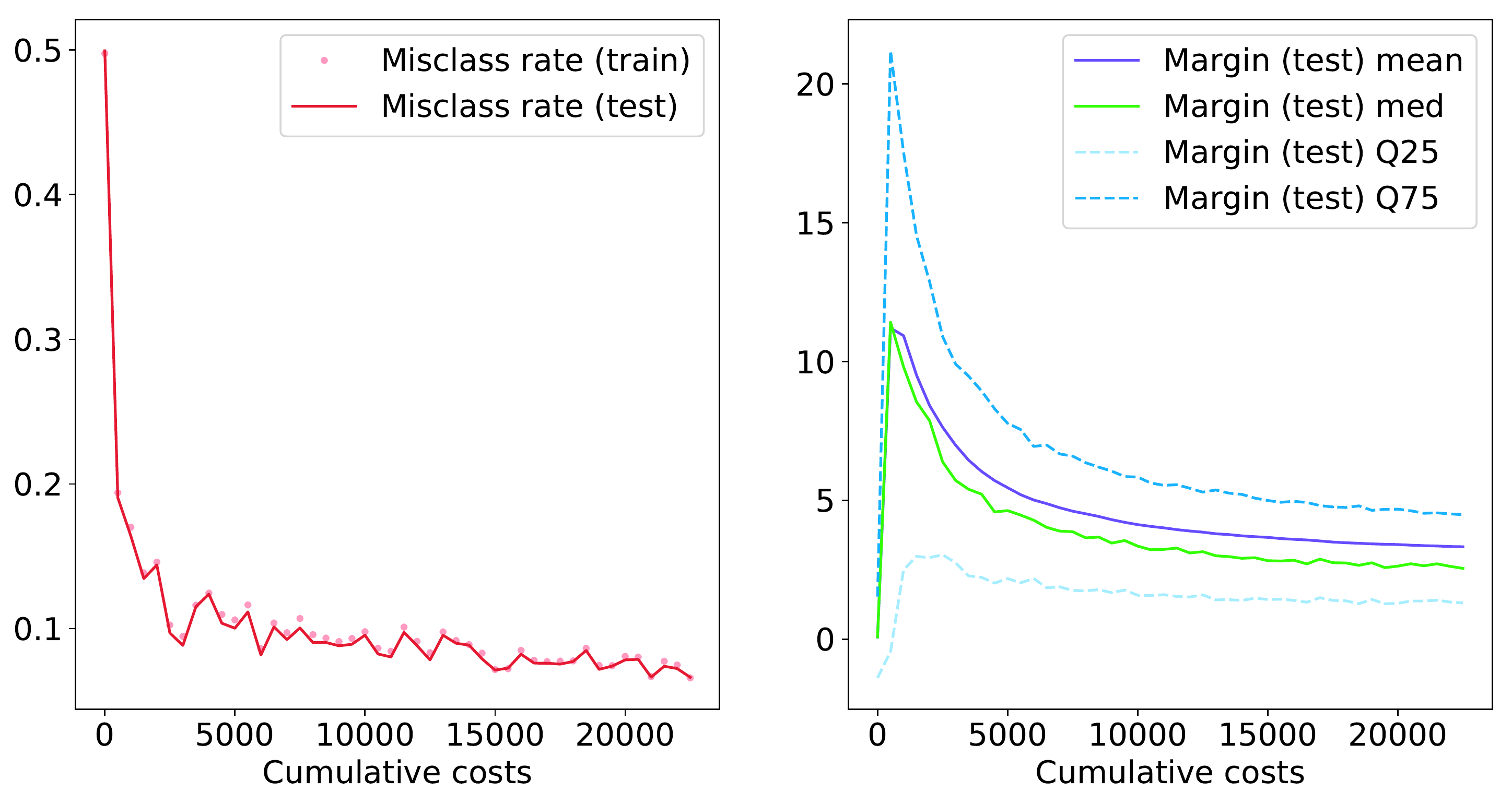}\includegraphics[width=0.5\textwidth]{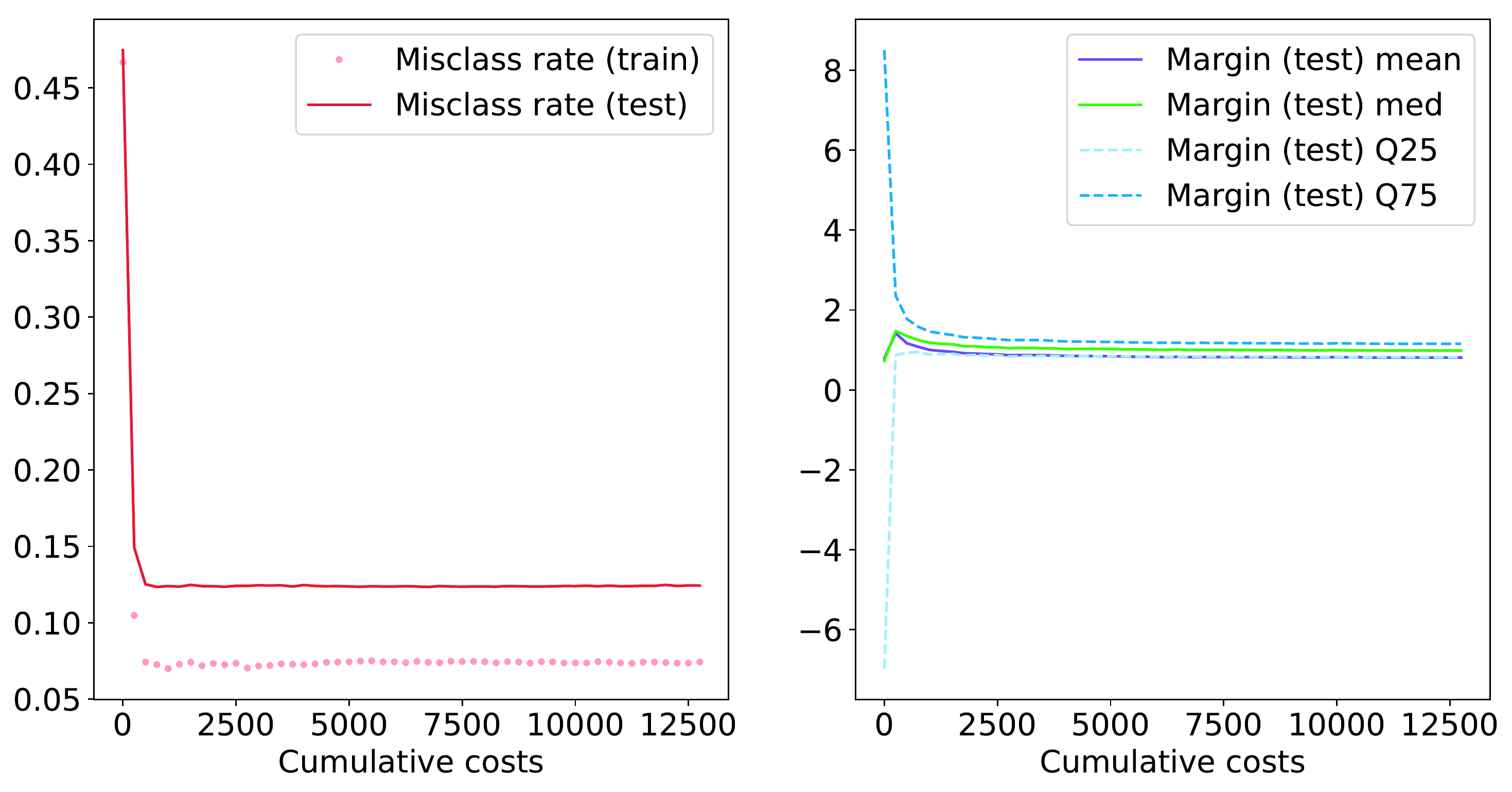}
\caption{Top row: Algorithm \ref{algo:mainGD}. Bottom row: Pegasos.}
\label{fig:realtest_naive_scale_2}
\end{figure}

\paragraph{Benchmark data tests: scaling and regularization}

Next, we look at the impact of a fixed scale, determined by observed data, as follows. Each run of Algorithm \ref{algo:mainGD} starts with $s=1$ fixed just as in the previous tests, but after a pre-fixed number of steps, updates the scale just once, to take a value of $s \geq \sqrt{nv_{X}/(2\lambda\log(\delta^{-1}))}$ (see Lemma \ref{lem:pointwise_accuracy}), where $v_{X}$ is approximated using the 75th quantile of the empirical distribution induced by $\{|y_{i}\,\langle \wwhat_{(t)}, \xx_{i} \rangle|: i \in [n]\}$. This time, we intentionally under-regularize, setting $\lambda$ at less than 1/100th of the best setting found in the previous tests.

Representative results are given in Figure \ref{fig:realtest_scaled_once}. When highly under-regularized, \textit{and} without scaling, the learning algorithm just wanders about, overwhelmed by the variance of the per-iteration sub-sampling; when the procedure is left to run like this, a good solution can rarely be found before the step size grows small, highly inefficient. On the other hand, using the simple data-driven scaling procedure just described to fix a ``safe'' value of $s$, we find that the learning algorithm is almost immediately accelerated, and in less time essentially catches up with the performance achieved under the best regularization possible. This is extremely encouraging, as it suggests that a safe, inexpensive, automated scaling procedure can make up for our lack of knowledge about the ideal regularization parameter, allowing for potentially significant savings in hyper-parameter exploration.

\begin{figure}[h]
\centering
\includegraphics[width=0.25\textwidth]{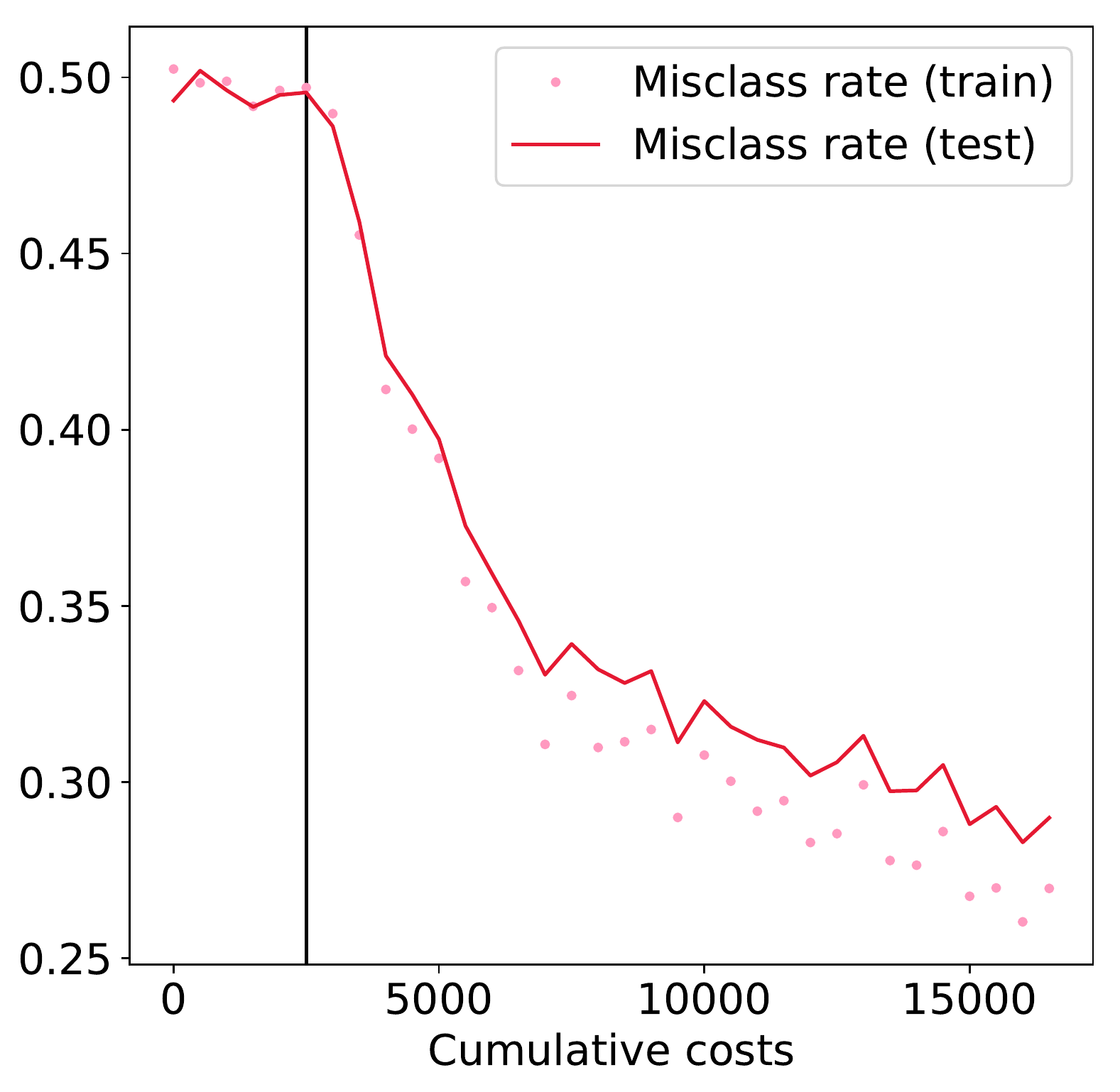}\includegraphics[width=0.25\textwidth]{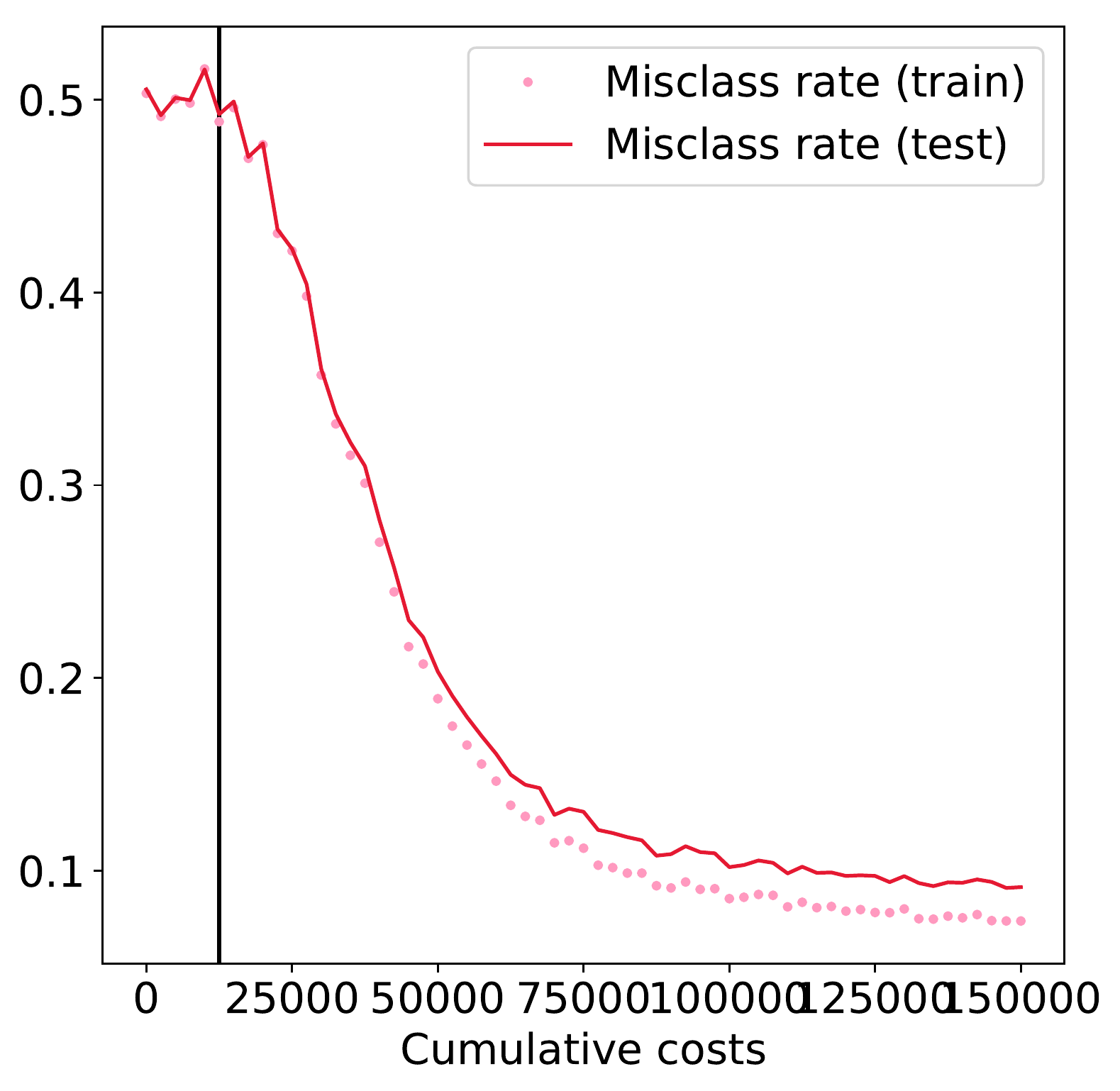}\includegraphics[width=0.25\textwidth]{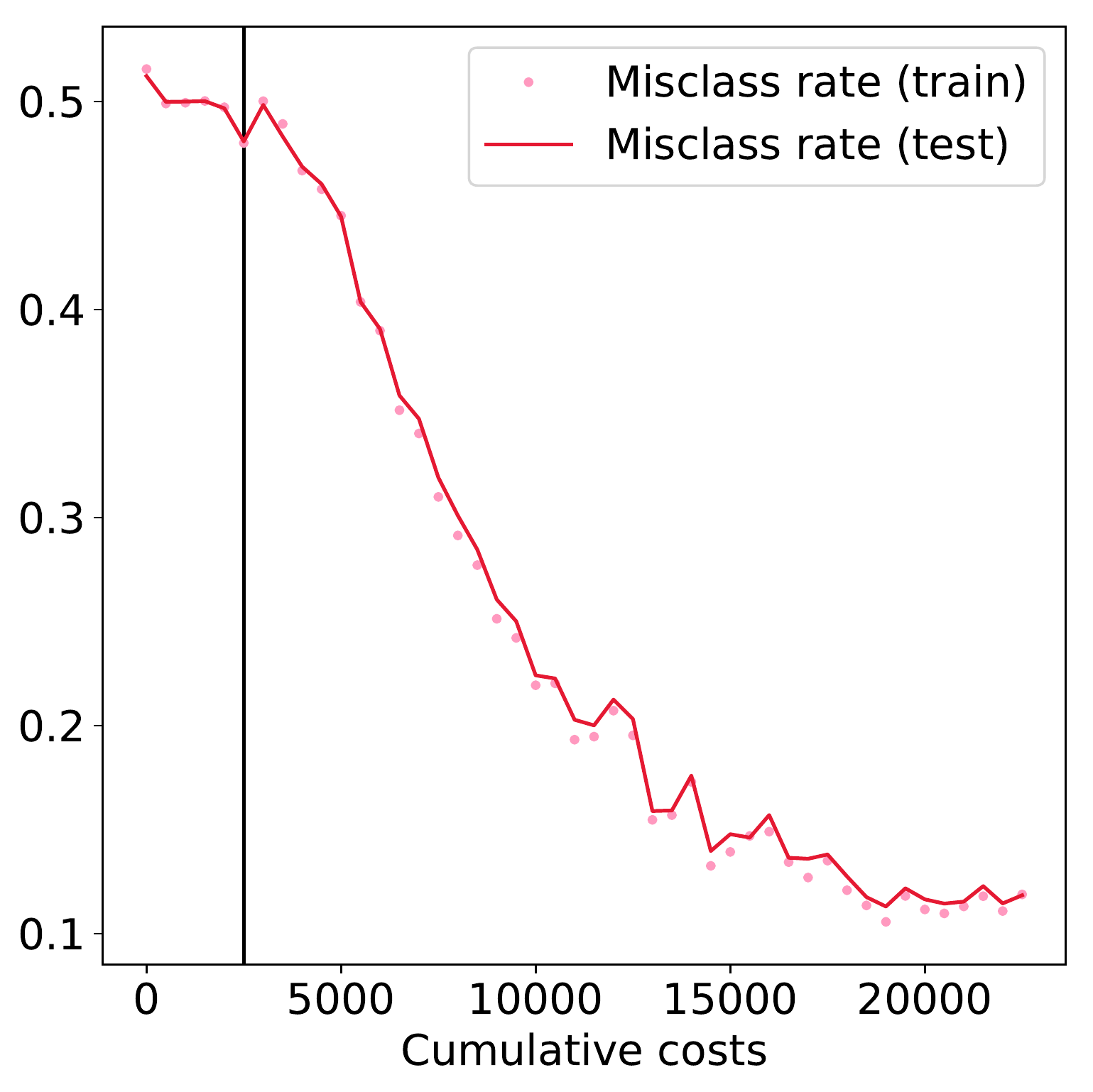}\includegraphics[width=0.25\textwidth]{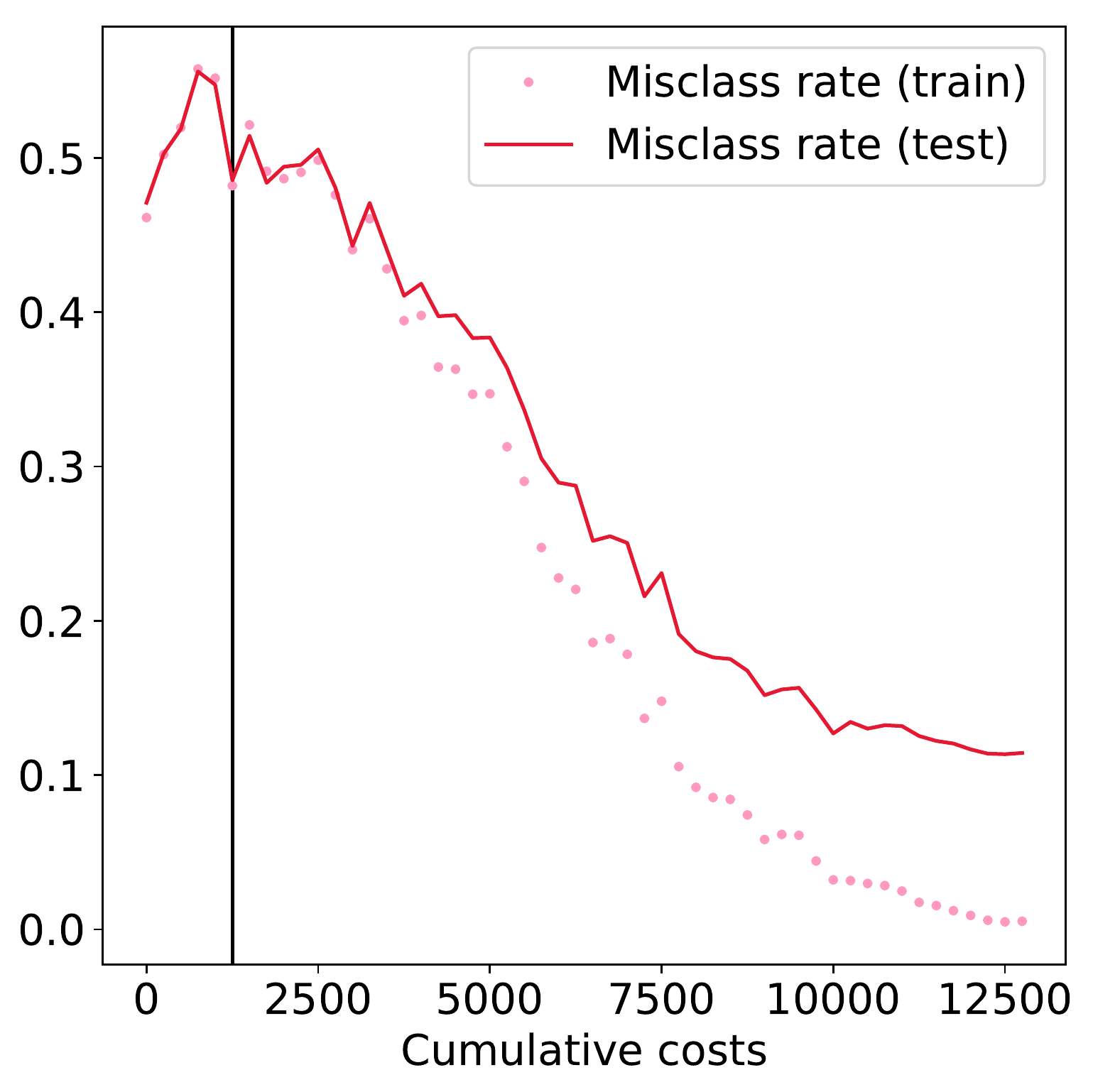}
\caption{Algorithm \ref{algo:mainGD} with data-based $s$ setting starting from the point marked by a black vertical line. From left to right, \textsc{cov}, \textsc{digit5}, \textsc{protein} (all $\lambda = 10^{-5}$), and \textsc{sido} ($\lambda = 10^{-3}$).}
\label{fig:realtest_scaled_once}
\end{figure}

\section{Concluding remarks}\label{sec:conclusion}

In this paper, we introduced and analyzed a new learning algorithm which, via a new convex loss with re-scaling, lets us pursue stronger guarantees for the resulting margin distribution (and classifier) than are possible with the traditional hinge loss. This allows us to bridge the gap between inference and computation, since strong learning guarantees are available for Algorithm \ref{algo:mainGD}, which is readily implemented in practice. Empirical tests confirmed that the algorithm basically behaves as we would expect, and that even with naive parameter settings, appropriate re-scaling on the back end allows our procedure to match or exceed the performance of well-known competitors.

\newpage

\bibliographystyle{refs/apalike}
\bibliography{refs/refs_margins.bib}

\newpage

\section{Technical appendix}\label{sec:appendix}

\subsection{Preliminaries}\label{sec:tech_prelims}

Here we put together few standard technical results that are utilized in the main proofs.

\begin{lem}
Let $f:\RR^{d} \to \RR$ be continuously differentiable, convex, and $l$-smooth. Then, we have
\begin{align}
\label{eqn:smoothness_property_1}
f(\uu) - f(\vv) & \leq \frac{l}{2}\|\uu - \vv\|^{2} + \langle f^{\prime}(\vv), \uu-\vv \rangle\\
\label{eqn:smoothness_property_2}
\|f^{\prime}(\uu) - f^{\prime}(\vv)\|^{2} & \leq 2l \left( f(\uu) - f(\vv) - \langle f^{\prime}(\vv), \uu-\vv \rangle \right).
\end{align}
\end{lem}
\begin{proof}
Given in chapter 2 of \citet{nesterov2004ConvOpt}.
\end{proof}

\begin{lem}\label{lem:smooth_surrogate}
The surrogate risk $\rnew(h)$ defined in (\ref{eqn:risk_surrogate}), for $h(\xx) = \langle \ww, \xx \rangle$, $\ww \in \RR^{d}$, is $l$-smooth with coefficient $l=\exx\|\xx\|^{2}=v_{X}$.
\end{lem}
\begin{proof}
Assuming the order of integration and differentiation can be reversed, one can write $\rnew^{\prime}$ as
\begin{align*}
\rnew^{\prime}(\ww) = -s \, \exx \rho^{\prime}\left( \frac{\gamma - y \langle \ww, \xx \rangle}{s} \right) y \xx.
\end{align*}
It follows that for arbitrary $\ww_{1}, \ww_{2} \in \RR^{d}$ we have
\begin{align*}
\| \rnew^{\prime}(\ww_{1}) - \rnew^{\prime}(\ww_{2}) \| & \leq s \, \exx \|\xx\| \left| \rho^{\prime}\left( \frac{\gamma - y \langle \ww_{1}, \xx \rangle}{s} \right) - \rho^{\prime}\left( \frac{\gamma - y \langle \ww_{2}, \xx \rangle}{s} \right) \right|\\
& \leq s \, \exx \|\xx\| \left|\frac{\langle \ww_{2}-\ww_{1}, \xx \rangle}{s}\right|\\
& \leq \|\ww_{2}-\ww_{1}\| \exx \|\xx\|^{2}
\end{align*}
where we utilized the property that $\rho^{\prime}$ is 1-Lipschitz. This implies that
\begin{align}\label{eqn:smooth_surrogate}
\| \rnew^{\prime}(\ww_{1}) - \rnew^{\prime}(\ww_{2}) \| \leq l \|\ww_{1} - \ww_{2}\|, \quad \ww_{1},\ww_{2} \in \RR^{d}
\end{align}
with coefficient $l = \exx \|\xx\|^{2}$, namely $\rnew$ is $\exx \|\xx\|^{2}$-smooth.
\end{proof}


\begin{lem}[Confidence interval for sample mean of sub-Gaussian random vector]\label{lem:subgaussian_samplemean_concineq}
Let $\xx$ be a random vector taking values in $\RR^{d}$, with the sub-Gaussian property
\begin{align*}
\exx \exp(a \langle \uu, \xx - \exx\xx \rangle ) \leq \exp(c a^{2} \langle \uu, \Sigma_{X} \uu \rangle ), \quad a \geq 0
\end{align*}
for some constant $c>0$ and $\Sigma_{X} \defeq \exx(\xx - \exx\xx)(\xx - \exx\xx)^{T}$. Given $n$ independent copies of $\xx$, denoted $\xx_{1},\ldots,\xx_{n}$, write $\overbar{\xx} \defeq n^{-1}\sum_{i=1}^{n} \xx_{i}$. Then with probability no less than $1-\delta$, we have
\begin{align*}
\|\overbar{\xx} - \exx \xx\| \leq 2 \sqrt{\frac{c \|\Sigma_{X}\| \log(\delta^{-1})}{n}}.
\end{align*}
\end{lem}
\begin{proof}[Proof of Lemma \ref{lem:subgaussian_samplemean_concineq}]
We use the Chernoff extension of Markov's inequality to establish exponential tails for the deviation of the sample mean from its expectation, a standard technique \citep{boucheron2013a}. For real-valued random variable $z \geq 0$, taking any $b > 0$ we have $b \, I\{z \geq b\} \leq z \, I\{z \geq b\}$ almost surely. Integrating both sides implies $b \, \prr\{z \geq b\} \leq \exx z I\{z \geq b\} \leq \exx z$, using the non-negativity of $z$ for the latter inequality. Thus $\prr\{z \geq b\} \leq \exx z / b$, the classic Markov inequality. For non-decreasing function $f(z) \geq 0$, this naturally extends via $\prr\{z \geq b\} \leq \prr\{ f(z) \geq f(b) \}$ to $\prr\{z \geq b\} \leq \exx f(z)/f(b)$, now for any real-valued random variable $z$. When $\exx z = 0$, setting $f(z) = z^{2}$ yields the special case of Chebyshev's inequality. Chernoff's inequality follows from the special case of $f(z) = \exp(a z)$, for $a > 0$, with the form
\begin{align*}
\prr\{z \geq b\} \leq e^{-a b} \exx \exp(a z).
\end{align*}
If the moment generating function of $z$ is not finite, then of course these bounds are vacuous, but in the sub-Gaussian case we have easily manipulated upper bounds. In our setup we have $z = \langle \uu , \xx - \exx\xx \rangle$, and by our hypothesis we have for any $\|\uu\| = 1$ that
\begin{align*}
\prr\{\langle \uu , \xx - \exx\xx \rangle \geq b\} & \leq e^{-a b} \exp(c a^{2} \langle \uu, \Sigma_{X}\uu \rangle)\\
& \leq \exp\left(c a^{2} \|\Sigma_{X}\| - a b \right)
\end{align*}
where $\|\Sigma_{X}\|$ denotes the $\ell_{2}$-induced matrix norm, equivalent to the spectral norm, i.e., the largest singular value of $\Sigma_{X}$ \citep{horn2013MA2nd}. Since this holds for any $a > 0$, this upper bound can be made as tight as possible when we set $a = b / (2c\|\Sigma_{X}\|)$, resulting in
\begin{align*}
\prr\{\langle \uu , \xx - \exx\xx \rangle \geq b\} \leq \exp\left(-\frac{b^{2}}{4c\|\Sigma_{X}\|}\right).
\end{align*}
For the special case of $\uu = (\xx - \exx\xx)/\|\xx - \exx\xx\|$, we have $\langle \uu , \xx - \exx\xx \rangle = \|\xx - \exx\xx\|$, yielding the same bound for $\prr\{\|\xx - \exx\xx\| \geq b\}$ as a special case.

Finally, for the sample mean, we note that
\begin{align*}
\langle \uu , \overbar{\xx} - \exx\xx \rangle = \frac{1}{n} \sum_{i=1}^{n}\left\langle \uu, (\xx_{i} - \exx \xx) \right\rangle.
\end{align*}
Plugging this in to our Chernoff equality, 
\begin{align*}
\prr\{\langle \uu , \overbar{\xx} - \exx\xx \rangle \geq b\} & \leq e^{-a nb} \prod_{i=1}^{n}\exp(c a^{2} \langle \uu, \Sigma_{X}\uu \rangle)\\
& \leq \exp\left(n c a^{2}\|\Sigma_{X}\| - a n b \right).
\end{align*}
Once again optimizing with respect to $a$, and setting $\uu = (\overbar{\xx} - \exx\xx)/\|\overbar{\xx} - \exx\xx\|$ as noted above, we have
\begin{align*}
\prr\{\|\overbar{\xx} - \exx\xx\| \geq b\} \leq \exp\left(-\frac{nb^{2}}{4c\|\Sigma_{X}\|}\right)
\end{align*}
which implies the desired result.
\end{proof}

\subsection{Proofs of results in the main text}\label{sec:tech_proofs}

\begin{proof}[Derivation of (\ref{eqn:newloss_rho})]
Considering the definition of $\psi$, for the case of $|u| \leq \sqrt{2}$, we have an indefinite integral
\begin{align*}
\int \psi(u) \, du = \frac{u^{2}}{2} - \frac{u^{4}}{24} + C_{1},
\end{align*}
and since we would like a loss function taking a value of zero at $u=0$, we set $C_{1}=0$. For the case of $u>\sqrt{2}$, we have
\begin{align*}
\int \psi(u) \, du = \frac{2\sqrt{2}}{3}u + C_{2}.
\end{align*}
Since we would like this integral to be continuous on $\RR$, we must have
\begin{align*}
\frac{u^{2}}{2} - \frac{u^{4}}{24} = \frac{2\sqrt{2}}{3}u + C_{2}
\end{align*}
when $u = \sqrt{2}$. Setting $C_{2} = -1/2$ achieves this.
\end{proof}

\begin{proof}[Proof of Remark \ref{rmk:unsatisfactory_learner}]
For simplicity, consider instance space $\XX = \RR$. As an intuitive model, consider $\HH_{\text{ray}}$, the set of all classifiers defined by rays in the ``left'' direction. That is, each $h \in \HH_{\text{ray}}$ takes the form
\begin{align*}
h(x;\alpha) = I\{x \leq \alpha\} - I\{x > \alpha\}
\end{align*}
for some $\alpha \in \RR$. Upon the underlying distribution, break up the input space into three segments, $(-\infty, \alpha_{l}^{\ast}]$. $(\alpha_{l}^{\ast}, \alpha_{u}^{\ast})$, $[\alpha_{l}^{\ast},\infty)$, with probabilities
\begin{align*}
\prr\{ x <= \alpha_{l}^{\ast} \} = \prr\{ x >= \alpha_{u}^{\ast} \} = 1/3.
\end{align*}
It follows that $\prr\{ x \in (\alpha_{l}^{\ast},\alpha_{u}^{\ast}) \} = 1/3$ as well. Furthermore, assume that the labeling of pair $(x,y)$ is done as
\begin{align*}
x \mapsto y = \begin{cases}
1, & x \notin (\alpha_{l}^{\ast}, \alpha_{u}^{\ast})\\
-1, & x \in (\alpha_{l}^{\ast}, \alpha_{u}^{\ast})
\end{cases}.
\end{align*}
In this situation, given the probabilities, it is evident that in terms of minimizing the classification error $\exx I\{y \neq h(x)\}$, the optimal choice is to select $h(\cdot;\alpha_{l}^{\ast})$, in which case
\begin{align*}
\exx I\{y \neq h(x;\alpha_{l}^{\ast})\} = \prr\{ x >= \alpha_{l}^{\ast} \} = 1/3.
\end{align*}
This gives us a lower bound on performance, namely
\begin{align*}
\exx I\{y \neq h(x)\} \geq 1/3, \quad \forall \, h \in \HH_{\text{ray}}.
\end{align*}
Now, in the limiting case of $\rho(u)=u^{2}$, say we have
\begin{align*}
\hstar \in \argmin_{h \in \HH_{\text{ray}}} \, \exx \left( \gamma - y \, h(\xx) \right)^{2}
\end{align*}
for a pre-fixed value of $\gamma > 1/3$, something we are free to do. Note that as $\prr\{ y \, h(x) <= 0 \} = \prr\{ y \neq h(x) \} \geq 1/3$, we have
\begin{align*}
\exx y \, \hstar(x) & = \prr\{ y \, \hstar(x) > 0 \} - \prr\{ y \, \hstar(x) <= 0 \}\\
& \leq \frac{2}{3} - \frac{1}{3}\\
& = \frac{1}{3}\\
& < \gamma.
\end{align*}
Since we know that
\begin{align*}
\exx y \, \hstar(x) = \argmin_{\gamma} \, \exx \left( \gamma - y \, \hstar(\xx) \right)^{2},
\end{align*}
it follows that the margin level $\gamma$ need not provide a reliable measure of the location of the distribution of $y \, \hstar(x)$ over a random draw from the underlying data distribution.
\end{proof}

\begin{proof}[Proof of Proposition \ref{prop:empirical_gamma_control}]
We begin with a sufficient condition for $\gamma$ to equal $\est(h)$,
\begin{align*}
\sum_{i=1}^{n} \rho^{\prime}\left( \gamma - y_{i} \, h(\xx_{i}) \right) = 0.
\end{align*}
This function is bounded on $\RR$ by $\pm B$, where $B \defeq \rho^{\prime}(\sqrt{2})$. For clarity, write $a_{i} \defeq y_{i} \, h(\xx_{i})$ for $i \in [n]$. Assume without loss of generality that $n>1$ is odd and $a_{i} \leq a_{i+1}$ for $i \in [n-1]$. Writing $m \defeq (n+1)/2$, the median value is $a_{m}$. Obviously, taking any scale such that
\begin{align*}
0 < s < \frac{|a_{m} - a_{i}|}{\sqrt{2}}, \quad i \neq m
\end{align*}
it follows immediately that
\begin{align*}
\sum_{i=1}^{n} \rho^{\prime}\left( a_{m} - a_{i} \right) = 0
\end{align*}
since the $m$th summand is zero, the first $(n-1)/2$ summands equal $\sqrt{2}$, and the last $(n-1)/2$ summands equal $-\sqrt{2}$, canceling each other out. Thus for any $s$ small enough, the median is a valid solution. Extending this to the case of $n$ even is straightforward. Writing $m \defeq n/2$ now, since $\rho^{\prime}(u) = -\rho^{\prime}(-u)$, it follows that $\rho^{\prime}_{s}(\gamma-a_{m}) + \rho^{\prime}_{s}(\gamma-a_{m+1}) = 0$ when we set $\gamma = (a_{m} + a_{m+1})/2$. Looking at the sum over $\{\rho^{\prime}(\gamma - a_{i})\}_{i \in [n]}$ There are $(n-2)/2$ terms no less than these two middle terms, and $(n-2)/2$ terms no greater than them. Just as before, taking $s > 0$ small enough, the former will equal $\sqrt{2}$ and the latter $-\sqrt{2}$, once again canceling each other out and leaving the median as a valid solution. This proves part 1 of the hypothesis.

As for the empirical mean case (part 2), we use a more general result, taken from \citet{holland2017a}:
\begin{lem}
Let $x$ be an arbitrary random variable with distribution $\nu$. Assuming $\exx_{\nu}|x|^{3} < \infty$, it follows that defining
\begin{align*}
\tstar \defeq \argmin_{\theta \in \RR} \exx_{\nu} \rho_{s}(\theta - x)
\end{align*}
the deviation can be controlled as
\begin{align*}
|\tstar - \exx_{\nu}x| \leq cs^{-2}, \quad s > 0
\end{align*}
for constant $c>0$.
\end{lem}
\noindent Substituting underlying data distribution for $\nu$, and $y \, h(\xx)$ for $x$, and considering the analogous
\begin{align*}
\gamma^{\ast}(h) \defeq \argmin_{\gamma \in \RR} \exx \rho_{s}(\gamma - y \, h(\xx))
\end{align*}
it immediately follows that
\begin{align*}
|\gamma^{\ast}(h) - \exx y \, h(\xx)| = O\left(\frac{1}{s^{2}}\right)
\end{align*}
for any valid distribution (i.e., where the third moment condition holds). This holds for the case of the empirical distribution $P_{n}(A) \defeq n^{-1} \sum_{i=1}^{n} I\{ y_{i} \, h(\xx_{i}) \in A\}$, and plugging in $P_{n}$ for $\nu$ we have that $\gamma^{\ast}(h) = \est(h)$, and obtain part 2.
\end{proof}

\begin{proof}[Proof of Proposition \ref{prop:stability}]
This stability property follows from basic properties of the function $\rho$, as follows. Since $h \in \HH$ is pre-fixed, we suppress it in the notation. Write $\{q_{i}\}_{i=1}^{n}$ and $\{q_{i}^{\prime}\}_{i=1}^{n}$ for the margins $y\,h(\xx)$ evaluated on the two data sets of interest (original and modified). By definition, on the pre-modification data set, we have
\begin{align*}
\sum_{i=1}^{n} \rho^{\prime}\left(\frac{\est-q_{i}}{s}\right) = 0.
\end{align*}
Assume that the $j$th index is the one where $\zz_{j} \neq \zz_{j}^{\prime}$, and thus where $q_{j} \neq q_{j}^{\prime}$. Without loss of generality, assume $q_{j} > q_{j}^{\prime}$. In the optimistic case, where $\est - q_{j} \geq \sqrt{2}$, this does not impact the estimator at all, and $\est = \est^{\prime}$. On the pessimistic side, the largest impact possible would be
\begin{align*}
\sum_{i=1}^{n} \rho^{\prime}\left(\frac{\est-q_{i}^{\prime}}{s}\right) & = \rho^{\prime}\left(\frac{\est-q_{j}^{\prime}}{s}\right) - \rho^{\prime}\left(\frac{\est-q_{j}}{s}\right)\\
& \geq \rho^{\prime}(\sqrt{2}) - \rho^{\prime}(-\sqrt{2})\\
& = 2 \rho^{\prime}(\sqrt{2})\\
& = 4 \sqrt{2} / 3.
\end{align*}
Thus, in order to satisfy the first order condition
\begin{align*}
\sum_{i=1}^{n} \rho^{\prime}\left(\frac{\est^{\prime}-q_{i}^{\prime}}{s}\right) = 0,
\end{align*}
a shift from $\est$ to $\est^{\prime}$ must in the worst case make up $4\sqrt{2}/3$, by sufficiently decreasing $\est^{\prime}$. Define an index
\begin{align*}
\II \defeq \left\{ i \in [n]: \frac{|\est -q_{i}|}{s} \leq \frac{\sqrt{2}}{2} \right\}.
\end{align*}
The value $\sqrt{2}/2$ in the definition of $\II$ is arbitrary; any value less than $\sqrt{2}$ would work fine, but this allows for a straightforward argument. The points with $i \in \II$ give us a worst-case value for how far we must shift from $\est$ to $\est^{\prime}$, as follows. Write $m \defeq |\II|$, and note that if $m \geq 3$, a jump of width $\sqrt{(2/m)}$ from the edge of our ``good range'' of $[-\sqrt{2}/2, \sqrt{2}/2]$ in the slowest-changing direction (say positive side, without loss of generality) yields a slope of
\begin{align*}
\rho^{\prime\prime}\left( \frac{\sqrt{2}}{2} + \sqrt{\frac{2}{m}} \right) = \frac{1}{2} - \frac{2(\sqrt{m}+1)}{m}.
\end{align*}
By symmetry, the slope on the negative side is the same. It follows that
\begin{align*}
\rho^{\prime}\left(-\frac{\sqrt{2}}{2}\right) - \rho^{\prime}\left(-\frac{\sqrt{2}}{2}-\sqrt{\frac{2}{m}}\right) \geq D \defeq  \sqrt{\frac{2}{m}}\left(\frac{1}{2} - \frac{2(\sqrt{m}+1)}{m}\right).
\end{align*}
Shifting $\est$ such that
\begin{align*}
\frac{\est^{\prime}-q_{i}^{\prime}}{s} - \frac{\est-q_{i}^{\prime}}{s} = -\sqrt{\frac{2}{m}},
\end{align*}
which is to say setting $\est^{\prime} = \est - s\sqrt{(2/m)}$, we know that at the very least, each $i \in \II$ contributes $-D$ to the sum in the first-order optimality condition. That is to say, we have
\begin{align*}
\sum_{i=1}^{n} \rho^{\prime}\left(\frac{\est-q_{i}^{\prime}}{s}\right) - \sum_{i=1}^{n} \rho^{\prime}\left(\frac{\est^{\prime}-q_{i}^{\prime}}{s}\right) \geq mD.
\end{align*}
All that is left is to ensure $mD \geq 4\sqrt{2}/3$. Some basic arithmetic shows that $m \geq 24$ implies $mD \geq 4\sqrt{2}/3$. Thus, we conclude that with $s$ such that $m=|\II| \geq 24$, the true $\est^{\prime}$ for the modified set can in the worst case be no farther from $\est$ than $s\sqrt{(2/m)}$, concluding the proof. Assuming $|\II| \geq n/2 \geq 24$ yields the desired result as a special case.
\end{proof}

\begin{proof}[Proof of Lemma \ref{lem:pointwise_accuracy}]
Extending the results of \citet{catoni2012a}, as long as $\rho$ satisfies
\begin{align}\label{eqn:catoni_condition}
-\log(1-u+Cu^{2}) \leq \rho^{\prime}(u) \leq \log(1+u+Cu^{2}), \quad u \in \RR
\end{align}
then exponential tails on the empirical estimator's deviation can be obtained; there is nothing particularly special about $\rho$ in (\ref{eqn:newloss_rho}) except the computational convenience and ease of analysis. Given this inequality, Lemma 1 of \citet{holland2017a} implies that
\begin{align*}
\prr\left\{\frac{|\est(h) - \exx y\,h(\xx)|}{2} \leq \frac{C\vaa y\,h(\xx)}{s} + \frac{s\log(2\delta^{-1})}{n}\right\} \geq 1-\delta.
\end{align*}
For our setting, $\rho$ defined in (\ref{eqn:newloss_rho}) indeed satisfies (\ref{eqn:catoni_condition}), with $C=1/2$, which follows from Lemma 1 of \citet{catoni2017a}, where this function is studied in the context of robust vector mean estimates. Optimizing the upper bound with respect to $s>0$ and plugging in $C=1/2$ yields the desired result.
\end{proof}

\begin{proof}[Proof of Lemma \ref{lem:Psi_transform}]
We follow along with the now-standard framework set out by \citet{bartlett2006b}. For simplicity, we start with the special case of $s = 1$, where the loss function becomes $\varphi(u) = \rho(\gamma - u)$.

Next we put together the analytical machinery that will be used. First, the conditional expected $\varphi$-risk takes the form
\begin{align*}
\exx \left( \varphi(y \, h(\xx)) | \xx \right) = \eta \varphi(y\,h(\xx)) + (1-\eta) \varphi(y\,h(\xx))
\end{align*}
where $\eta$ denotes $\eta = \prr\{ y = 1\}$. A generalization of this quantity for arbitrary $\eta \in [0,1]$ is constructed as
\begin{align*}
C_{\eta}(u) \defeq \eta \varphi(u) + (1-\eta)\varphi(-u), \quad u \in \RR.
\end{align*}
The optimal value that this takes is denoted by
\begin{align*}
H(\eta) \defeq \inf_{u \in \RR} C_{\eta}(u).
\end{align*}
Denote any optimal value using $u^{\ast}$, namely any
\begin{align*}
u^{\ast} \in \argmin_{u \in \RR} C_{\eta}(u).
\end{align*}
If this value is indeed unique, then it makes sense to map $\eta \mapsto u^{\ast}(\eta)$. In relating $R$ and $R_{\varphi}$, our intuitive concern is the degree to which, on average, $\varphi(y\,h(\xx))$ can be small while $I\{h(\xx) \neq y\}$ remains non-zero. This notion is captured formally by the following nice quantity:
\begin{align*}
H^{-}(\eta) \defeq \inf \left\{ C_{\eta}(u): u(2\eta - 1) \leq 0 \right\}.
\end{align*}
Using the ``best (generalized) conditional $\varphi$-risk that can be achieved despite having different signs from $(2\eta-1)$.'' Of course, the word ``despite'' becomes appropriate when we replace $\eta$ with the conditional probability $\eta(\xx) = \prr\{y = 1 | \xx\}$, in which $\sign(2\eta(\xx)-1)$ is the Bayes decision rule for this classification task. For $\varphi$ to be a good surrogate, we would expect that $H^{-}$ should tend to be larger than $H$. If this was not the case, a small $\varphi$-risk could be achieved despite having mis-labeled some instances, which would immediately imply a small $R_{\varphi}$ but larger $R$. To ensure that this cannot happen, the condition put forward by \citet{bartlett2006b} is very lucid: call $\varphi$ \textit{classification-calibrated} if
\begin{align*}
H^{-}(\eta) > H(\eta), \quad \forall \, \eta \neq 1/2.
\end{align*}
The size of this gap is defined as
\begin{align*}
\widetilde{\Psi}(a) \defeq H^{-}\left(\frac{1+a}{2}\right) - H\left(\frac{1+a}{2}\right),
\end{align*}
and the Fenchel-Legendre bi-conjugate of $\widetilde{\Psi}$ is denoted by $\Psi$. It is this function that has the desirable properties that interest us. For one, via their Theorem 1, for any non-negative $\varphi$, any distribution on $\XX \times \{-1,1\}$ and measurable function $h$, we have
\begin{align*}
\Psi\left( R(h) - R^{\ast} \right) \leq R_{\varphi}(h) - R_{\varphi}^{\ast}.
\end{align*}
It is easy to characterize this classification-calibration in the convex case. If $\varphi$ is convex, then via their Theorem 2(1),
\begin{center}
$\varphi$ is classification calibrated $\iff$ $\varphi$ is differentiable at zero with $\varphi^{\prime}(0) < 0$.
\end{center}
Since our function is $(d/du)\varphi(u) = \rho(\gamma - u)(-1)$, for $\gamma > 0$ we have $\rho(\gamma) > 0$ and thus $(d/du)\varphi(0) < 0$ as desired. Thus our $\varphi$, being a convex function on $\RR$, is classification calibrated. Furthermore, via Theorem 2(2), the $\Psi$ function takes a particularly simple form:
\begin{align*}
\Psi_{1,\gamma}(a) = \varphi(0) - H_{1,\gamma}\left(\frac{1+a}{2}\right), \quad -1 \leq a \leq 1
\end{align*}
where $\varphi(0) = \rho(\gamma)$ gives us the expression from the hypothesis in the case of $s=1$. All that remains in order to obtain $\Psi_{1,\gamma}$ then is to compute $H_{1,\gamma}$ explicitly, which we carry out below.

Now, since both $\varphi(u)$ and $\varphi(-u)$ are convex functions of $u$, and $C_{\eta}(u)$ is a convex combination of these two, it follows that $C_{\eta}(u)$ is also convex. Furthermore, noting that both $u \to \infty$ and $u \to -\infty$ imply $C_{\eta}(u) \to \infty$. Thus a minimum clearly exists, and can be characterized by a first-order condition as follows. Taking the first derivative of $C_{\eta}(\cdot)$, we have
\begin{align*}
\frac{d}{du} C_{\eta}(u) = \eta \rho^{\prime}(\gamma - u)(-1) + (1-\eta) \rho^{\prime}(\gamma + u) = 0
\end{align*}
which using $\rho^{\prime}(-u) = (-1)\rho^{\prime}(u)$, can be equivalently stated as
\begin{align}\label{eqn:condition_1storder}
\frac{\rho^{\prime}(u - \gamma)}{\rho^{\prime}(u + \gamma)} = \frac{\eta - 1}{\eta}.
\end{align}
That is to say, for $\eta \in (0,1)$, any $u^{\ast}$ satisfying (\ref{eqn:condition_1storder}) will be a minimizer in that $C_{\eta}(u^{\ast}) \leq C_{\eta}(u)$ for all $u$.

It should be clear that the value of $\gamma$ plays an important role in finding the solution. Note that $\sqrt{2}$ is an important threshold here, since
\begin{align*}
u \geq \sqrt{2} \implies \rho^{\prime}(u) = \rho^{\prime}(\sqrt{2}) = \frac{2\sqrt{2}}{3}.
\end{align*}
On the ``left'' side as well, $u \leq -\sqrt{2}$ implies $\rho^{\prime}(u) = -\rho^{\prime}(\sqrt{2})$.

For the case of $\eta = 0$, we have $u^{\ast}=-\gamma$, and when $\eta = 1$ we have $u^{\ast}=\gamma$. This implies that $H_{1,\gamma}(0) = H_{1,\gamma}(1) = 0$. More generally, an obvious but important fact is that for any $\eta \in (0,1)$, any solution $u^{\ast}$ must fall on the open interval $(-\gamma,\gamma)$. This is because the right-hand side of (\ref{eqn:condition_1storder}) is always negative, but the left-hand side is negative if and only if $u-\gamma < 0 < u+\gamma$, equivalently $u \in (-\gamma,\gamma)$. Also, for the case of $\eta = 1/2$, we have that (\ref{eqn:condition_1storder}) is always satisfied by setting $u=0$, for which case we have
\begin{align*}
H_{1,\gamma}(1/2) = \frac{1}{2} \rho(\gamma - 0) + \frac{1}{2} \rho(\gamma + 0) = \rho(\gamma).
\end{align*}
This value, and thus the height of the peak of $H_{1,\gamma}(\cdot)$, changes as a function of $\gamma$. Let's proceed and look at evaluating $H_{1,\gamma}(\eta)$ for $\eta \in (0,1)$ when $\eta \neq 1/2$. We shall consider the following distinct settings:

\begin{enumerate}
\item $0 < \gamma \leq \sqrt{2}/2$

\item $\sqrt{2}/2 \leq \gamma < \sqrt{2}$

\item $\sqrt{2} \leq \gamma$
\end{enumerate}

Doing these one at a time, first consider $0 < \gamma \leq \sqrt{2}/2$. This case is simple, since for any solution $u^{\ast} \in (-\gamma,\gamma)$ we have that $u^{\ast} \pm \gamma \in [-\sqrt{2}, \sqrt{2}]$ and thus we can set $\rho^{\prime}(u) = u-u^{3}/6$, re-arrange equality (\ref{eqn:condition_1storder}), and solve for roots of the resulting cubic polynomial. The computations are quick, and writing 
\begin{align*}
P(u;a,b,c,d) \defeq a u^{3} + b u^{2} + c u + d
\end{align*}
and $\alpha \defeq (\eta - 1)/\eta$, the new condition is
\begin{align}
a & = 1-\alpha \nonumber\\
b & = -3 \gamma (1+\alpha) \nonumber\\
c & = 3(1-\alpha)(\gamma^{2} - 2) \nonumber\\
d & = (1+\alpha)(6\gamma - \gamma^{3}) \nonumber\\
\label{eqn:cond_doublecube}
P(u;a,b,c,d) & = 0.
\end{align}
Call this (\ref{eqn:cond_doublecube}) the \textit{double-cube} condition (see Figure \ref{fig:cond_doublecube}).
\begin{figure}[t]
\centering
\includegraphics[width=0.66\textwidth]{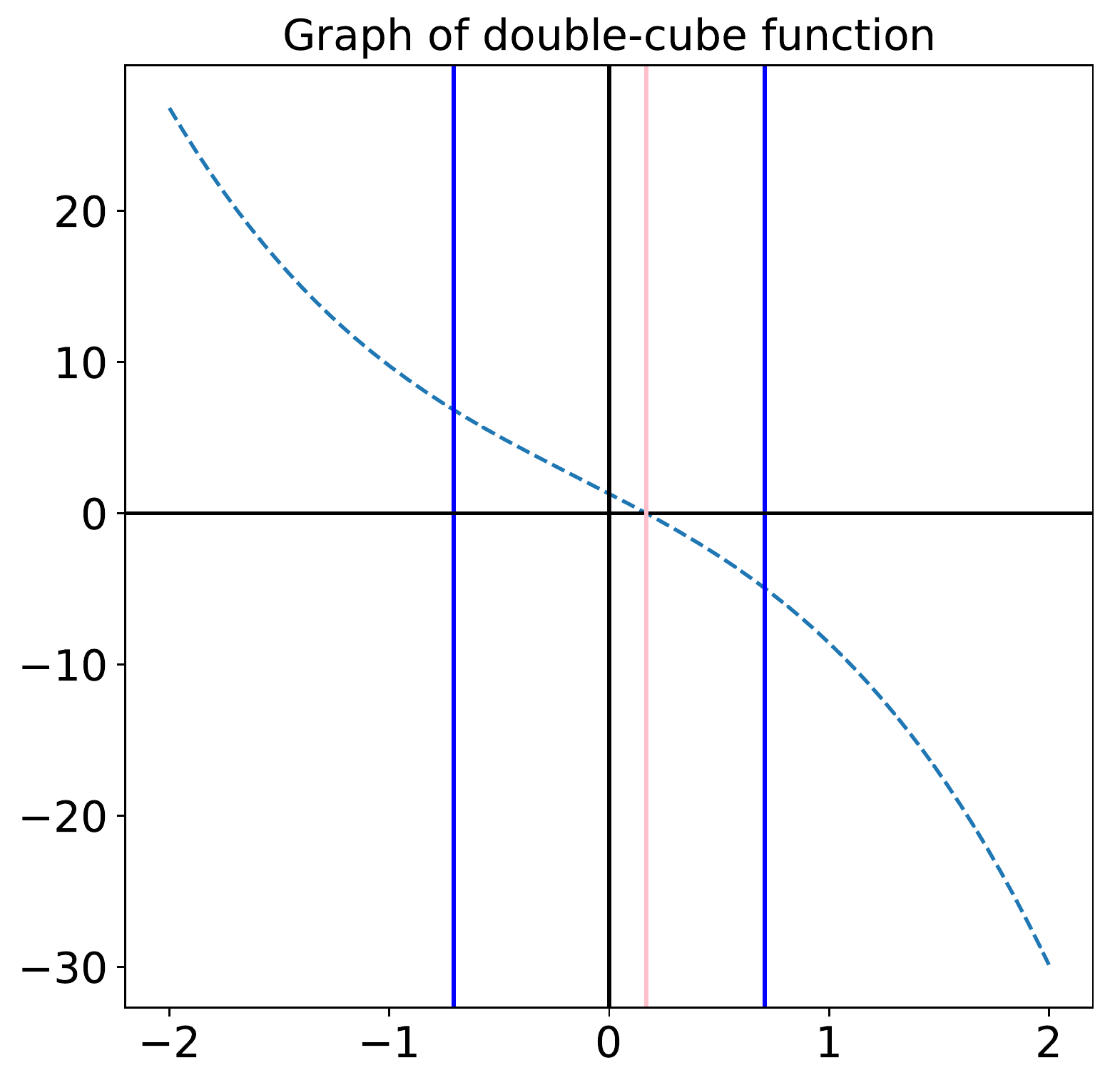}
\caption{Graph of the third-degree polynomial used in the double-cube condition (\ref{eqn:cond_doublecube}). Vertical blue lines denote $\pm \gamma$ (here $\gamma = \sqrt{2}/2$), and the vertical pink line denotes the root computed analytically.}
\label{fig:cond_doublecube}
\end{figure}
The discriminant of an arbitrary cubic polynomial is defined
\begin{align}\label{eqn:discriminant_defn}
\Delta \defeq 18abcd - 4b^{3}d + b^{2}c^{2} - 4ac^{3} - 27a^{2}d^{2},
\end{align}
and as long as $\Delta < 0$, the function $P(u)$ has only one real root, which can be computed analytically (see Appendix). Writing $u^{\ast}(\eta)$ for the real value satisfying $P(u^{\ast}(\eta))=0$ here, by plugging this into the original objective we get $H_{1,\gamma}(\eta) = C_{\eta}(u^{\ast}(\eta))$.

Next consider the case of $\sqrt{2}/2 < \gamma < \sqrt{2}$. This is the most complicated case. Writing $\delta \defeq |\sqrt{2} - \gamma|$, if a solution exists on the interval $[-\delta, \delta]$, then it will naturally satisfy the double-cube condition given above. If there is no solution on this interval, then depending on whether $\eta > 1/2$ or $\eta < 1/2$, the appropriate condition will respectively be
\begin{align*}
\rho^{\prime}(u - \gamma) - \rho^{\prime}(\sqrt{2})\frac{\eta - 1}{\eta} = 0
\end{align*}
or 
\begin{align*}
\rho^{\prime}(u + \gamma) + \frac{\eta}{\eta - 1}\rho^{\prime}(\sqrt{2}) = 0.
\end{align*}
Call these the \textit{minus} and \textit{plus} \textit{single-cube} conditions. Re-arranged into more explicit terms, we have respectively
\begin{align}
a & = 1 \nonumber\\
b & = -3\gamma \nonumber\\
c & = 3\gamma^{2}-6 \nonumber\\
d & = -6 \left( \gamma + \frac{\rho^{\prime}(\sqrt{2})}{\alpha} - \frac{\gamma^{3}}{6} \right) \nonumber\\
\label{eqn:cond_singlecube_minus}
P(u;a,b,c,d) & = 0
\end{align}
and
\begin{align}
a & = 1 \nonumber\\
b & = 3\gamma \nonumber\\
c & = 3\gamma^{2}-6 \nonumber\\
d & = 6 \left( \gamma + \alpha \rho^{\prime}(\sqrt{2}) - \frac{\gamma^{3}}{6} \right) \nonumber\\
\label{eqn:cond_singlecube_plus}
P(u;a,b,c,d) & = 0
\end{align}
Graphs of these polynomials are plotted in Figure \ref{fig:cond_singlecubes}.
\begin{figure}[t]
\centering
\includegraphics[width=1\textwidth]{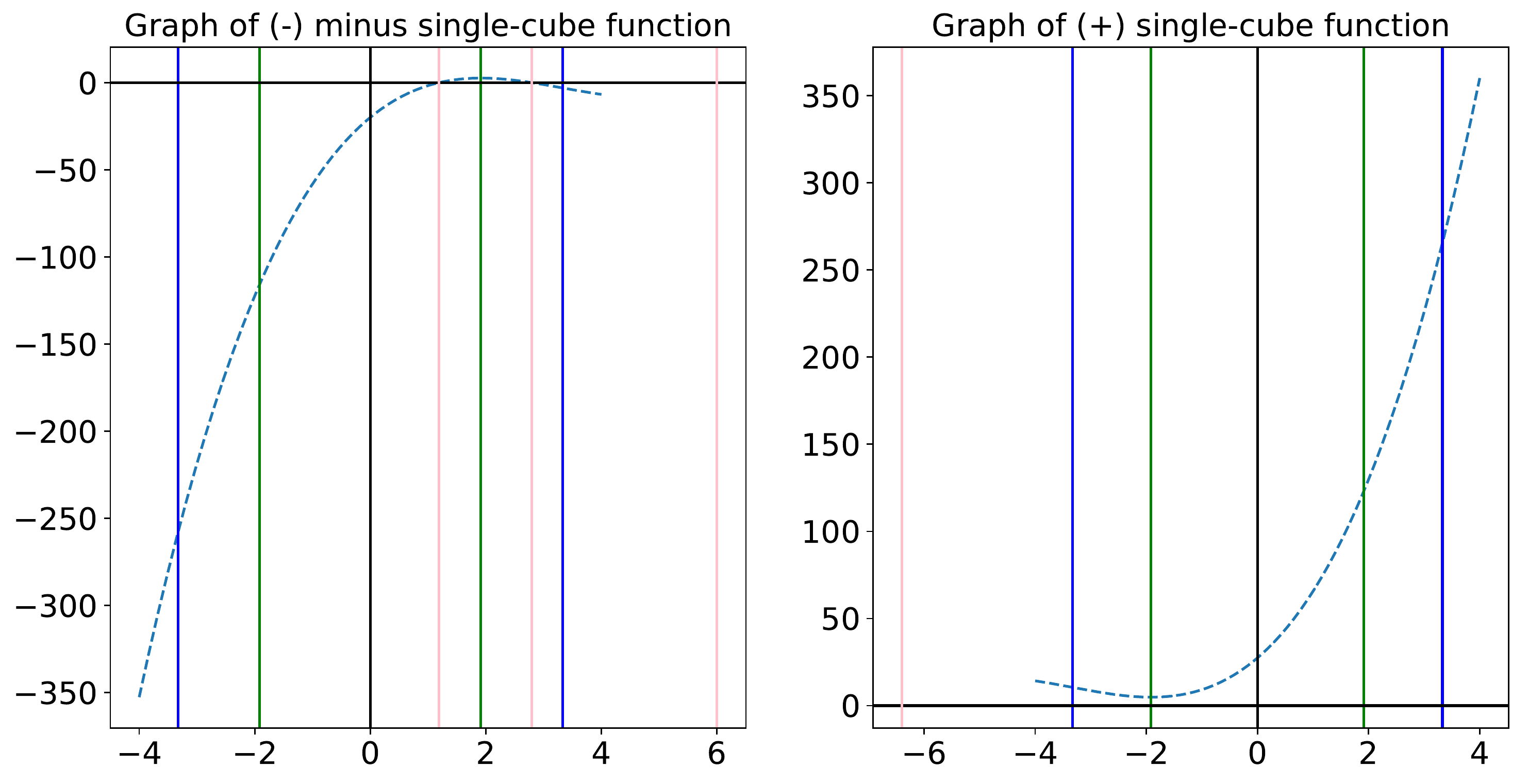}
\caption{Graph of the third-degree polynomials used in the single-cube conditions. The left figure corresponds to condition (\ref{eqn:cond_singlecube_minus}), and the right figure corresponds to condition (\ref{eqn:cond_singlecube_plus}). The vertical blue lines are again $\pm \gamma$ (with $\gamma = 2\sqrt{2}+1/2$ here), and the vertical green lines are $\pm \delta$.}
\label{fig:cond_singlecubes}
\end{figure}
Computationally determining which to use is straightforward. By the monotonicity of $\rho^{\prime}$, we can simply check the edge case $u = \sign(\eta - 1/2)\delta$. In the case of $\eta > 1/2$, noting that both the LHS and RHS are negative, if
\begin{align}\label{eqn:delta_cond_bigger}
\frac{\rho^{\prime}(\delta - \gamma)}{\rho^{\prime}(\delta + \gamma)} < \frac{\eta - 1}{\eta},
\end{align}
then the solution must be larger than $\delta$, and thus the minus single-cube condition will be sufficient. Else, the double-cube condition will provide a solution. On the other hand, when $\eta < 1/2$, if
\begin{align}\label{eqn:delta_cond_smaller}
\frac{\rho^{\prime}(-\delta - \gamma)}{\rho^{\prime}(-\delta + \gamma)} > \frac{\eta - 1}{\eta}, \text{ or more cleanly, } \frac{\rho^{\prime}(\delta - \gamma)}{\rho^{\prime}(\delta + \gamma)} < \frac{\eta}{\eta - 1},
\end{align}
then the solution must be below $-\delta$, and thus the plus single-cube condition will be sufficient. Else, the double-cube condition will provide a solution. This gives us a simple procedure for the current range of $\gamma$ values being considered\footnote{While there may be more than one real root of the cubic polynomials used in these conditions, there will not be more than one root in the range of $(\delta,\gamma)$ ($\eta > 1/2$ case) or $(-\gamma,-\delta)$ ($\eta < 1/2$ case).}, as follows:
\begin{itemize}
\item When $\eta > 1/2$:
\begin{itemize}
\item If (\ref{eqn:delta_cond_bigger}), then solve (\ref{eqn:cond_singlecube_minus}), take root falling in $(\delta, \gamma)$.
\item Else, solve (\ref{eqn:cond_doublecube}).
\end{itemize}
\smallskip
\item When $\eta < 1/2$:
\begin{itemize}
\item If (\ref{eqn:delta_cond_smaller}), then solve (\ref{eqn:cond_singlecube_plus}), take root falling in $(-\gamma,-\delta)$.
\item Else, solve (\ref{eqn:cond_doublecube}).
\end{itemize}
\end{itemize}

Finally, consider the case of $\sqrt{2} \leq \gamma$. This situation is simple: if $\eta > 1/2$, find solutions to the minus single-cube condition, and if $\eta < 1/2$, find solutions to the plus single-cube condition.

With all these conditions in place, it follows that for any $\gamma > 0$ and any $\eta \in [0,1]$, we can find a solution $u^{\ast}$ such that $C_{\eta}(u^{\ast}) = H_{1,\gamma}(\eta)$. It follows then that following the procedures outlined above, we can also compute $\Psi_{1,\gamma}(a) = \varphi(0)-H_{1,\gamma}((1+a)/2)$ for arbitrary $a \in [-1,1]$.

The consistency part of the lemma statement follows immediately from the calibration of $\varphi$, using \citet[Theorem 1(3)]{bartlett2006b}. Invertibility of $\Psi_{1,\gamma}$ follows from convexity of $\varphi$ and Lemma 2 of \citet{bartlett2006b}.

It remains only to extend these results to the case of arbitrary $s > 0$, namely the general loss function $\varphi(u) = s^{2} \, \rho((\gamma-u)/s)$. That $\varphi(u)$ is convex and that $\varphi^{\prime}(0) < 0$ under arbitrary $s>0$ is immediate. Furthermore, the first-order optimality condition becomes
\begin{align}\label{eqn:condition_1storder_general}
\frac{\eta-1}{\eta} = \frac{\rho^{\prime}((u-\gamma)/s)}{\rho^{\prime}((u+\gamma)/s)} = \frac{\rho^{\prime}((u/s)-(\gamma/s))}{\rho^{\prime}((u/s)+(\gamma/s))}.
\end{align}
Writing $\widetilde{\gamma} = \gamma/s$, note that using the exact same procedures outlined above, we can always find a $u^{\prime}$ such that
\begin{align*}
\frac{\eta-1}{\eta} = \frac{\rho^{\prime}(u^{\prime}-\widetilde{\gamma})}{\rho^{\prime}(u^{\prime}+\widetilde{\gamma})},
\end{align*}
which means that writing $u^{\ast} = s u^{\prime}$, we have that $u^{\ast}$ is a solution of (\ref{eqn:condition_1storder_general}). This means that for any $\gamma > 0$, $s > 0$, and $\eta \in [0,1]$, we can find a solution $u^{\ast}$ such that $C_{\eta}(u^{\ast}) = H_{s,\gamma}(\eta)$, which yields the general $\Psi_{s,\gamma}$ as
\begin{align*}
\Psi_{s,\gamma}(u) = s^{2} \, \rho\left(\frac{\gamma}{s}\right) - H_{s,\gamma}\left(\frac{1+u}{2}\right),
\end{align*}
concluding the proof.
\end{proof}

\begin{proof}[Proof of Theorem \ref{thm:riskbd}]
To keep notation clean, throughout this proof we denote the risk gradient by $\rgrad(\ww) \defeq \risk^{\prime}(\ww)$, the surrogate risk gradient by $\rnewgrad(\ww) \defeq \rnew^{\prime}(\ww)$, and the new loss gradient by $\gest(\ww) \defeq L^{\prime}(\ww;\gamma)$. By Lemma \ref{lem:Psi_transform}, we have that for any choice of $\ww \in \WW$,
\begin{align*}
\Psi_{s,\gamma}\left( R(\ww) - R^{\ast} \right) \leq R_{\varphi}(\ww) - R_{\varphi}^{\ast}.
\end{align*}
To control the right-hand side, note that by strong convexity $\wwstar$ is the unique minimum of $R_{\varphi}$, and so $R_{\varphi}(\ww) - R_{\varphi}^{\ast} = R_{\varphi}(\ww) - R_{\varphi}(\wwstar)$. Since $R_{\varphi}$ is smooth via Lemma \ref{lem:smooth_surrogate} with coefficient $v_{X}$, using the basic property (\ref{eqn:smoothness_property_1}) of smooth functions, we have
\begin{align*}
R_{\varphi}(\ww) - R_{\varphi}(\wwstar) \leq v_{X}\|\ww - \wwstar\|^{2}.
\end{align*}
It remains to control $\|\wwhat_{(t)} - \wwstar\|$, where $\wwhat_{(t)}$ is the output of a single iteration of the \textbf{for} loop in Algorithm \ref{algo:mainGD}. This can be broken up into computational and statistical elements, as follows. We can readily bound this distance from above as
\begin{align*}
\|\wwhat_{(t+1)} - \wwstar \| & = \| \wwhat_{(t)} - \alpha \, \gest(\wwhat_{(t)}) - \wwstar \|\\
& \leq \| \wwhat_{(t)} - \alpha \, \rnewgrad(\wwhat_{(t)}) - \wwstar \| + \alpha \| \gest(\wwhat_{(t)}) - \rnewgrad(\wwhat_{(t)}) \|.
\end{align*}
The initial equality follows immediately by design of Algorithm \ref{algo:mainGD} and the assumption that $\alpha_{(t)}=\alpha$. Using the triangle yields the upper bound, which is composed of two terms: the first term is the difference after doing one update of the ideal gradient descent routine (to minimize $\rnew$), and the second term is a statistical error term for the empirical mean estimate of the surrogate risk gradient vector.

For small enough step size $0 < \alpha < 2 / (\kappa + v_{X})$, the update improves on the previous error as
\begin{align*}
\| \wwhat_{(t)} - \alpha \, \rnewgrad(\wwhat_{(t)}) - \wwstar \|^{2} \leq \left(1 - \frac{2\alpha\kappa v_{X}}{\kappa + v_{X}}\right)\|\wwhat_{(t)} - \wwstar\|^{2}.
\end{align*}
Writing $\beta \defeq 2\kappa v_{X} / (\kappa + v_{X})$, we have that
\begin{align}\label{eqn:proof_driver_1}
\|\wwhat_{(t+1)} - \wwstar \| \leq \sqrt{1-\alpha\beta}\|\wwhat_{(t)} - \wwstar\| + \alpha \| \gest(\wwhat_{(t)}) - \rnewgrad(\wwhat_{(t)}) \|.
\end{align}
This deals with the computational error part. Now for the statistical error part, namely the accuracy of the $\gest \approx \rnewgrad$ approximation. Writing $\mv{b}(\ww) \defeq -\rho((\gamma-y \langle \ww,\xx \rangle )/s) \, y \, \xx$ and multiplying this by $s$, the sub-Gaussianity assumption of \ref{asmp:sub_gaussian} gives us that for any $\ww \in \RR^{d}$ and $a \geq 0$,
\begin{align*}
\exx \exp(a \langle \uu, s\,\mv{b}(\ww) - \exx s\,\mv{b}(\ww) \rangle ) & = \exx \exp((as) \langle \uu, \mv{b}(\ww) - \exx \mv{b}(\ww) \rangle )\\
& \leq \exp(c (as)^{2} \langle \uu, \Sigma(\ww) \uu \rangle )\\
& = \exp( (cs^{2})a^{2} \langle \uu, \Sigma(\ww) \uu \rangle )
\end{align*}
for all $\|\uu\| = 1$, where $\Sigma(\ww)$ is the covariance matrix of $\mv{b}(\ww)$, and $c > 0$ is any constant such that the sub-Gaussian property holds. Now, suppressing $\ww$ from the notation for readability, noting that
\begin{align*}
\Sigma & = \exx (\mv{b}-\exx \mv{b})(\mv{b}-\exx \mv{b})^{T}\\
& = \exx \mv{b}\mv{b}^{T} - (\exx \mv{b})(\exx \mv{b})^{T},
\end{align*}
and using the positive semi-definiteness of $(\exx \mv{b})(\exx \mv{b})^{T}$, since for all $\uu$ we have
\begin{align*}
\langle \uu, (\exx \mv{b}\mv{b}^{T} - \Sigma)\uu \rangle = \langle \uu, (\exx \mv{b})(\exx \mv{b})^{T}\uu \rangle \geq 0,
\end{align*}
for each $a \geq 0$ we can then bound
\begin{align*}
\exx \exp(a \langle \uu, \mv{b} - \exx \mv{b} \rangle ) & \leq \exp(cs^{2} a^{2} \langle \uu, \Sigma \uu \rangle )\\
& \leq \exp(cs^{2} a^{2} \langle \uu, (\exx \mv{b}\mv{b}^{T}) \uu \rangle )\\
& \leq \exp(cs^{2} a^{2} \|\exx \mv{b}\mv{b}^{T}\| )\\
& \leq \exp(cs^{2} a^{2} \exx \|\mv{b}\mv{b}^{T}\| )\\
& \leq \exp(cs^{2} a^{2} \rho^{\prime}(\sqrt{2})^{2}\exx \|\xx\xx^{T}\| )
\end{align*}
With these inequalities, writing
\begin{align*}
\mv{b}_{i}(\ww) \defeq -s\,\rho^{\prime}\left(\frac{\gamma - y_{i} \langle \ww, \xx_{i} \rangle}{s}\right) \, y_{i} \, \xx_{i}, \quad i \in [n]
\end{align*}
for any fixed $\ww$ and noting that $\gest(\ww) = n^{-1}\sum_{i=1}^{n}\mv{b}_{i}(\ww)$, we can leverage Lemma \ref{lem:subgaussian_samplemean_concineq} to prove that the ``bad event''
\begin{align}\label{eqn:badevent_prob}
\EE(\ww) \defeq \left\{ \| \gest(\ww) - \rnewgrad(\ww) \| > 2s \sqrt{\frac{c\rho^{\prime}(\sqrt{2})^{2}\exx\|\xx\xx^{T}\|\log(\delta^{-1})}{n}} \right\}
\end{align}
has probability $\prr\EE(\ww) \leq \delta$.

Next, we must deal with the fact that in running Algorithm \ref{algo:mainGD}, the $\wwhat_{(t)}$ for all $t>0$ will be random and dependent on the sample. In general, there is not much choice but to pursue uniform bounds, namely high-probability events that hold over all $\ww \in \WW$. To do this is straightforward with an $\epsilon$-cover of $\WW$. Since $\WW$ is a compact subset of $\RR^{d}$ by assumption \ref{asmp:model_compact}, it follows that the size of an $\epsilon$-cover in the usual norm is bounded as $N_{\epsilon} \leq (3\diameter/2\epsilon)^{d}$ \citep{kolmogorov1993SelectWorks3}. Denote the centers of the $\epsilon$ balls covering $\WW$ by $\{\wwtil_{1},\ldots,\wwtil_{N_{\epsilon}}\}$. Given any arbitrary $\ww \in \WW$, write $\wwtil = \wwtil(\ww)$ for the center closest to $\ww$, which by definition satisfies $\|\ww - \wwtil\| \leq \epsilon$. The statistical error can be bounded above by
\begin{align}\label{eqn:ineq_3errors}
\| \gest(\ww) - \rnewgrad(\ww) \| & \leq \|\gest(\ww) - \gest(\wwtil)\| + \|\rnewgrad(\ww) - \rnewgrad(\wwtil)\| + \|\gest(\wwtil) - \rnewgrad(\wwtil)\|.
\end{align}
We want to take the supremum of both sides with respect to $\ww \in \WW$. Let's take it term by term.

Starting with the first term, by the $1$-Lipschitz property of $\rho^{\prime}$, it follows immediately that we can bound
\begin{align*}
\|\mv{b}_{i}(\ww) - \mv{b}_{i}(\wwtil)\| \leq \frac{\|\xx_{i}\|\|\ww - \wwtil\|}{s} \leq \frac{\|\xx_{i}\|\epsilon}{s}.
\end{align*}
This implies that
\begin{align}\label{eqn:bound_error1}
\|\gest(\ww) - \gest(\wwtil)\| \leq \left(\frac{\epsilon}{s}\right) \frac{s}{n} \sum_{i=1}^{n} \|\xx_{i}\|^{2} \leq \frac{\epsilon \exx \|\xx\|^{2}}{\delta} = \frac{\epsilon v_{X}}{\delta}
\end{align}
on an event of probability no less than $1-\delta$, where we have simply used Chebyshev's inequality to obtain tail bounds. Since regardless of what $\ww$ we choose, the corresponding $\wwtil$ will be no farther than $\epsilon$, this (\ref{eqn:bound_error1}) gives us a uniform bound.

For the second term, we just use the $v_{X}$-smoothness of $\rnew$, shown in Lemma \ref{lem:smooth_surrogate}. This implies 
\begin{align}\label{eqn:bound_error2}
\|\rnewgrad(\ww) - \rnewgrad(\wwtil)\| \leq v_{X} \|\ww - \wwtil\| \leq v_{X} \epsilon
\end{align}
again for arbitrary choice of $\ww \in \WW$.

Finally, for any fixed $\wwtil \in \{\wwtil_{1},\ldots,\wwtil_{N_{\epsilon}}\}$, we can bound the third term using (\ref{eqn:badevent_prob}). Making the dependence of $\wwtil$ on $\ww$ explicit for clarity, the critical fact is that
\begin{align*}
\sup_{\ww \in \WW} \left\|\gest(\wwtil(\ww)) - \rnewgrad(\wwtil(\ww)) \right\| = \max_{k \in [N_{\epsilon}]} \left\|\gest(\wwtil_{k}) - \rnewgrad(\wwtil_{k}) \right\|.
\end{align*}
The ``good event'' of interest here is the event in which the bad event does not occur at any of the $\epsilon$-cover centers, that is
\begin{align*}
\EE_{+} = \left(\bigcap_{k \in [N_{\epsilon}]} \EE(\wwtil_{k}) \right)^{c},
\end{align*}
where $(\cdot)^{c}$ denotes the complement event. It thus follows that taking a union bound, we have that with probability no less than $1-\delta$, we can uniformly bound as
\begin{align}\label{eqn:bound_error3}
\|\gest(\wwtil(\ww)) - \rnewgrad(\wwtil(\ww))\| \leq 2s \sqrt{\frac{c\rho^{\prime}(\sqrt{2})^{2}\exx\|\xx\xx^{T}\|\log(N_{\epsilon}\delta^{-1})}{n}}, \quad \forall \, \ww \in \WW.
\end{align}
Putting these three bounds together, and taking unions over the good events required for the first and third terms, we have with probability no less than $1-2\delta$ that
\begin{align*}
\sup_{\ww \in \WW} \| \gest(\ww) - \rnewgrad(\ww) \| \leq \frac{\epsilon v_{X}}{\delta} + v_{X} \epsilon + 2s \sqrt{\frac{c\rho^{\prime}(\sqrt{2})^{2}\exx\|\xx\xx^{T}\|\log(N_{\epsilon}\delta^{-1})}{n}}.
\end{align*}
Setting $\epsilon = \delta / \sqrt{n}$, this simplifies to
\begin{align*}
\sup_{\ww \in \WW} \| \gest(\ww) - \rnewgrad(\ww) \| \leq \frac{(1+\delta)v_{X}}{\sqrt{n}} + \frac{\varepsilon^{\ast}}{\sqrt{n}}.
\end{align*}
where we have defined
\begin{align}\label{eqn:gradient_stat_error}
\varepsilon^{\ast} \defeq 2s \sqrt{c\rho^{\prime}(\sqrt{2})^{2}\exx\|\xx\xx^{T}\|(d\log(3\sqrt{n}(2\delta)^{-1}) + \log(\delta^{-1}))}.
\end{align}

On the good event $\EE_{+}$, then denoting $\varepsilon = (\varepsilon^{\ast} + (1+\delta)v_{X})/\sqrt{n}$, for \emph{all} steps $t$ we can re-write (\ref{eqn:proof_driver_1}) as
\begin{align*}
\|\wwhat_{(t+1)} - \wwstar \| \leq \sqrt{1-\alpha\beta}\|\wwhat_{(t)} - \wwstar\| + \alpha \varepsilon.
\end{align*}
Assuming the algorithm is run for $T$ updates, then with some straightforward algebra we can unfold and clean up the recursion such that
\begin{align*}
\|\wwhat_{(T)} - \wwstar \| \leq (\sqrt{1-\alpha\beta})^{T} \|\wwhat_{(0)}-\wwstar\| + \frac{2\varepsilon}{\beta}.
\end{align*}
Let us connect all the inequalities now. We can bound the excess surrogate risk as
\begin{align*}
R_{\varphi}(\wwhat_{(T)}) - R_{\varphi}^{\ast} \leq \left( (1-\alpha\beta)^{T} \|\wwhat_{(0)}-\wwstar\|^{2} + \frac{4}{\beta^{2}n}\left((1+\delta)v_{X} + \varepsilon^{\ast}\right)^{2} \right) v_{X}
\end{align*}
which via the first inequality of this proof using $\Psi_{s,\gamma}$, yields the desired result.
\end{proof}

\subsection{Explanation of root-finding function \texttt{getroot}}\label{sec:tech_getroot}

We have uploaded some demonstrative software for the root-finding sub-routine needed for computation of $\Psi_{s,\gamma}$ to a public repository.\footnote{Available at: \url{https://github.com/feedbackward/catcube}.} The core routine is captured in a function called \texttt{getroot}, using a very simple strategy, which we describe below.

The cubic polynomials considered are equations of the form
\begin{align}\label{eqn:cubic_original}
a u^{3} + b u^{2} + cu + d = 0.
\end{align}
Recall the discriminant $\Delta$ given in (\ref{eqn:discriminant_defn}). There are a few basic settings to consider, as below.
\begin{itemize}
\item If $\Delta < 0$, then there is only one real root (the rest are complex).

\item If $\Delta = 0$, then all roots are real, but we have multiple roots.

\item If $\Delta > 0$, then all roots are real, and distinct.
\end{itemize}

For the case of $\Delta < 0$, the traditional solution approach is as follows. Defining two new quantities
\begin{align*}
\Delta_{0} & \defeq b^{2} - 3ac\\
\Delta_{1} & \defeq 2b^{3} - 9abc + 27a^{2}d
\end{align*}
the key value for computing roots is the following
\begin{align*}
C = \left( \frac{\Delta_{1} \pm \sqrt{\Delta_{1}^{2} - 4\Delta_{0}^{3}}}{2} \right)^{1/3}.
\end{align*}
Assuming that $C$ is known, then roots are computed as
\begin{align*}
u^{\ast} = - \frac{1}{3a} \left( b + C + \frac{\Delta_{0}}{C} \right).
\end{align*}
Naturally, if $C$ is real, then so is the resulting $u^{\ast}$. Computationally, how do we go about getting a real version? This is extremely straightforward. Let's take the addition case. Consider the condition
\begin{align*}
\Delta_{1} + \sqrt{\Delta_{1}^{2} - 4\Delta_{0}^{3}} \geq 0.
\end{align*}
If this condition holds, then we can just compute as-is. If this condition fails to hold, then taking the cube root in many programming languages will lead to a complex number. To get a real number when the above condition fails, just compute
\begin{align*}
C = (-1)\left( \frac{|\Delta_{1} + \sqrt{\Delta_{1}^{2} - 4\Delta_{0}^{3}}|}{2} \right)^{1/3}.
\end{align*}
With a real-valued $C$ in hand, $u^{\ast}$ immediately follows.

Next, consider the case of $\Delta = 0$. This case is very simple. This scenario also sub-divides, based on the value of $\Delta_{0}$. If $\Delta_{0} = 0$, then the root is a ``triple'' root, and takes the form
\begin{align*}
u_{T}^{\ast} = - \frac{b}{3a}.
\end{align*}
If $\Delta_{0} \neq 0$, then we have two roots, a ``double'' root $u_{D}^{\ast}$ and a ``single'' root $u_{S}^{\ast}$, with the forms
\begin{align*}
u_{D}^{\ast} & = \frac{9ad - bc}{2\Delta_{0}}\\
u_{S}^{\ast} & = \frac{4abc - 9a^{2}d - b^{3}}{a\Delta_{0}}.
\end{align*}

Finally, consider the case of $\Delta > 0$. For elegant computations, we make use of the trigonometric method pioneered by F.~Vi\`{e}te. The starting point is a trigonometric identity, as follows:
\begin{align}\label{eqn:viete_trig_identity}
\cos 3x = 4 \cos^{3} x - 3 \cos x.
\end{align}
To prove this is straightforward. Making use of elementary trigonometric identities, observe first that
\begin{align*}
\cos 3x = \cos(2x + x) = \cos 2x \cos x - \sin 2x \sin x.
\end{align*}
Looking at each of the terms individually,
\begin{align*}
\cos 2x \cos x & = (2\cos^{2} x - 1) \cos x\\
 & = 2\cos^{3} x - \cos x\\
\sin 2x \sin x & = 2\sin x \cos x \sin x\\
& = 2\cos x \sin^{2} x\\
& = 2\cos x (1-\cos^{2} x)
\end{align*}
Taking the difference of the two new forms gives the desired identity (\ref{eqn:viete_trig_identity}). With this identity now at our disposal, we proceed with cleaning up the cubic equation. Dividing out $a$, and replacing $u$ with $v - b/(3a)$, note that this cleans up to
\begin{align}\label{eqn:depressed_cubic}
v^{3} + p v + q = 0
\end{align}
where
\begin{align*}
p & = \frac{3ac - b^{2}}{3a^{2}}\\
q & = \frac{2b^{3} - 9abc + 27a^{2}d}{27a^{3}}.
\end{align*}
Note that since we are assuming $\Delta > 0$ for the original cubic equation (\ref{eqn:cubic_original}), which implies three distinct real roots for (\ref{eqn:cubic_original}), it follows that (\ref{eqn:depressed_cubic}) also has three distinct real roots. This can only happen when its discriminant is positive, which is to say when
\begin{align}\label{eqn:depressed_inequality}
-4p^{3} - 27q^{2} > 0.
\end{align}
Note that this implies
\begin{align*}
p^{3} < -\frac{27q^{2}}{4} \leq 0.
\end{align*}
This implies that $p < 0$, otherwise the cube of $p$ would necessarily be non-negative. Moving forward, considering the trigonometric identity (\ref{eqn:viete_trig_identity}), the desired form of our cubic equation is $4z^{3} - 3z = e$, where $|e| \leq 1$ so that it falls in the range of the cosine function. To aid computations, let us introduce a couple more variables and coefficients. Set $k$ such that $p = -3k^{2}$, multiply by $4$, and replace $v$ with $rz$, where $r$ is a coefficient to be defined shortly. Doing so, we have
\begin{align*}
0 & = 4(rz)^{3} + 4 (-3k^{2}) rz + 4 q\\
& = 4z^{3} - \frac{12k^{3}}{r^{2}} z + \frac{4q}{r^{3}}.
\end{align*}
Setting $r = 2k$, we can clean up into the following equation
\begin{align}\label{eqn:trig_finalform}
4z^{3} - 3z = -\frac{q}{2k^{3}},
\end{align}
which is the desired form, as long as the right-hand side has absolute value no greater than unity. Fortunately, this is immediately true from our assumptions. To see this, first observe
\begin{align*}
\left( \frac{q}{2k^{3}} \right)^{2} = \frac{-27q^{2}}{4p^{3}} = \frac{27q^{2}}{4|p|^{3}}
\end{align*}
and recall that from (\ref{eqn:depressed_inequality}) and the fact that $p<0$, it follows that
\begin{align*}
0 & < -27q^{2} - 4p^{3} = -27q^{2} + 4|p|^{3}
\end{align*}
which implies
\begin{align*}
1 > \frac{27q^{2}}{4|p|^{3}} = \left( \frac{q}{2k^{3}} \right)^{2}.
\end{align*}
Thus, we have that in our current case of $\Delta > 0$, the right-hand side of (\ref{eqn:trig_finalform}) indeed falls on the interval $(-1,1)$. As such, this means that there exists an angle $x^{\ast}$ such that plugging $\cos x^{\ast}$ into the polynomial (\ref{eqn:trig_finalform}), we have
\begin{align*}
4 \cos^{3} x^{\ast} - 3\cos x^{\ast} = -\frac{q}{2k^{3}}.
\end{align*}
Then using the key identity (\ref{eqn:viete_trig_identity}), it follows that
\begin{align*}
\cos(3 x^{\ast}) = -\frac{q}{2k^{3}}, \text{ implying } x^{\ast} = \frac{1}{3} \arccos \left(-\frac{q}{2k^{3}}\right).
\end{align*}
So, we have that $\cos x^{\ast}$ solves (\ref{eqn:trig_finalform}). Note that since this function has period $2\pi$, it holds that
\begin{align*}
\cos(3x^{\ast}) = \cos(2\pi+3x^{\ast}) = \cos(2\pi-3x^{\ast})
\end{align*}
which after plugging in to (\ref{eqn:viete_trig_identity}), yields
\begin{align*}
\cos\left( 3\left(\frac{2\pi}{3} \pm x^{\ast}\right) \right) = 4\cos^{3}\left(\frac{2\pi}{3} \pm x^{\ast}\right) - 3\cos\left(\frac{2\pi}{3} \pm x^{\ast}\right) = -\frac{q}{2k^{3}}.
\end{align*}
That is to say, the following values are all solutions to (\ref{eqn:trig_finalform}):
\begin{align*}
z_{1}^{\ast} = \cos(x^{\ast}), \quad z_{2}^{\ast} = \cos\left(\frac{2\pi}{3}+x^{\ast}\right), \quad z_{3}^{\ast} = \cos\left(\frac{2\pi}{3}-x^{\ast}\right).
\end{align*}
With these values in hand, all that remains is to backtrack to the roots of the original cubic polynomial of interest. For any $z^{\ast}_{j}$, this is done as
\begin{align*}
y^{\ast}_{j} & = 2k z^{\ast}_{j}\\
u^{\ast}_{j} & = y^{\ast}_{j} - \frac{b}{3a}
\end{align*}
where $j=1,2,3$. To summarize the case of finding roots when $\Delta > 0$, the basic computational procedure is as below.
\begin{enumerate}
\item From original polynomial, $(a,b,c,d) \mapsto (p,q)$.

\item From new polynomial, $(p,q) \mapsto k \mapsto e$, where $e \defeq -q/(2k^{3})$.

\item From final polynomial, $e \mapsto x^{\ast} \mapsto (z^{\ast}_{1},z^{\ast}_{2},z^{\ast}_{3})$.

\item Backtrack over the roots as $z^{\ast}_{j} \mapsto y^{\ast}_{j} \mapsto u^{\ast}_{j}$ for $j=1,2,3$.

\item Return $(u^{\ast}_{1},u^{\ast}_{2},u^{\ast}_{3})$ as roots of (\ref{eqn:cubic_original}).
\end{enumerate}
This concludes our exposition of the content of \texttt{getroot}.

\end{document}